\documentclass{article}
\usepackage{lmodern}
\PassOptionsToPackage{numbers,sort&compress}{natbib}
\usepackage{natbib}

\usepackage[utf8]{inputenc} 
\usepackage[T1]{fontenc}    

\usepackage{url}            
\usepackage{booktabs}       
\usepackage{amsfonts}       
\usepackage{nicefrac}       
\usepackage{microtype}      
\usepackage{xcolor}         
\usepackage{microtype}      
\usepackage{amsmath, amsthm, amssymb} 
\usepackage{algorithm}
\usepackage{algorithmic}
\usepackage{bm}
\usepackage{mathtools}
\usepackage{xcolor}
\usepackage{graphicx}
\usepackage{wrapfig,lipsum}
\usepackage{thmtools, thm-restate}
\usepackage{mathrsfs} 
\usepackage{bbold}
\usepackage{multirow}

\usepackage{graphicx}
\usepackage{subcaption}
\usepackage{mwe} 

\usepackage[shortlabels]{enumitem}
\usepackage{cancel}

\newcommand{\algo}{ROVER}
\newcommand{\ALGO}{Reward-free pretraining via Occupancy coVERage maximization (ROVER)}

\definecolor{dgreen}{rgb}{0.00,0.49,0.00}
\definecolor{dblue}{rgb}{0,0.08,0.75}

\RequirePackage[colorlinks,citecolor=dgreen,urlcolor=dblue,linkcolor=dblue]{hyperref}

\newcommand{\sink}{S_\varepsilon}

\newcommand{\bpr}{\begin{proof}}
\newcommand{\epr}{\end{proof}}
\newcommand{\be}{\begin{equation}}
\newcommand{\ee}{\end{equation}}

\newcommand{\bd}{\begin{definition}}
\newcommand{\ed}{\end{definition}}

\newcommand{\bi}{\begin{itemize}}
\newcommand{\ei}{\end{itemize}}

\newcommand{\ba}{\begin{ass}}
\newcommand{\ea}{\end{ass}}

\newcommand{\br}{\begin{remark}}
\newcommand{\er}{\end{remark}}

\newcommand{\bp}{\begin{proposition}}
\newcommand{\ep}{\end{proposition}}

\newcommand{\blm}{\begin{lemma}}
\newcommand{\elm}{\end{lemma}}

\newcommand{\bt}{\begin{theorem}}
\newcommand{\et}{\end{theorem}}

\newcommand{\bcor}{\begin{corollary}}
\newcommand{\ecor}{\end{corollary}}

\newcommand{\bex}{\begin{example}}
\newcommand{\eex}{\end{example}}

\usepackage[capitalize]{cleveref}

\crefname{assumption}{Assumption}{Assumptions}
\crefname{equation}{Eq.}{Eqs.}
\crefname{figure}{Fig.}{Fig.}
\crefname{table}{Table}{Tables}
\crefname{section}{Sec.}{Sec.}
\crefname{theorem}{Thm.}{Thm.}
\crefname{lemma}{Lemma}{Lemmas}
\crefname{corollary}{Cor.}{Cor.}
\crefname{example}{Example}{Examples}
\crefname{remark}{Remark}{Remarks}
\crefname{algorithm}{Alg.}{Algorightms}

\crefname{appendix}{Appendix}{Appendices}

\makeatletter
\def\@endtheorem{\endtrivlist}
\makeatother

\usepackage{etoolbox}

\declaretheorem[name=Theorem,refname=Theorem]{theorem}
\declaretheorem[name=Lemma,sibling=theorem]{lemma}

\declaretheorem[name=Proposition,refname=Proposition,sibling=theorem]{proposition}
\declaretheorem[name=Remark,sibling=theorem]{remark}
\declaretheorem[name=Corollary,refname=Corollary,sibling=theorem]{corollary}
\declaretheorem[name=Definition,refname=Definition]{definition}

\declaretheorem[name=Example]{example}

\let\mathsf\relax    
\DeclareRobustCommand{\mathsf}[1]{\text{\normalfont\sffamily#1}}

\newcommand{\msf}[1]{\mathsf{#1}}

\newcommand{\X}{{\mathcal X}}
\newcommand{\A}{{\mathcal{A}}}

\newcommand{\hh}{{\mathcal{H}}}

\renewcommand{\paragraph}[1]{~\newline\noindent{\bf #1.}}

\crefname{assumption}{Assumption}{Assumptions}
\crefname{equation}{}{}
\Crefname{equation}{Eq.}{Eqs.}
\crefname{figure}{Fig.}{Fig.}
\crefname{table}{Table}{Tables}
\crefname{section}{Sec.}{Sec.}
\crefname{theorem}{Thm.}{Thm.}
\crefname{proposition}{Prop.}{Prop.}
\crefname{fact}{Fact}{Facts}
\crefname{lemma}{Lemma}{Lemmas}
\crefname{corollary}{Cor.}{Cor.}
\crefname{example}{Example}{Examples}
\crefname{remark}{Remark}{Remarks}
\crefname{algorithm}{Alg.}{Algorightms}

\crefname{appendix}{Appendix}{Appendices}

\newcommand{\eps}{\varepsilon}

\renewcommand{\sink}{{\msf{S}_\eps}}

\AfterEndEnvironment{restatable}{\noindent\ignorespacesafterend}
\AfterEndEnvironment{theorem}{\noindent\ignorespacesafterend}
\AfterEndEnvironment{remark}{\noindent\ignorespacesafterend}
\AfterEndEnvironment{example}{\noindent\ignorespacesafterend}
\AfterEndEnvironment{assumption}{\noindent\ignorespacesafterend}
\AfterEndEnvironment{lemma}{\noindent\ignorespacesafterend}
\AfterEndEnvironment{proof}{\noindent\ignorespacesafterend}
\AfterEndEnvironment{corollary}{\noindent\ignorespacesafterend}
\AfterEndEnvironment{proposition}{\noindent\ignorespacesafterend}
\AfterEndEnvironment{definition}{\noindent\ignorespacesafterend}

\makeatletter
\def\@endtheorem{\endtrivlist}
\makeatother

\usepackage[capitalize]{cleveref}

\crefname{assumption}{Asm.}{Asm.}
\crefname{equation}{}{}
\Crefname{equation}{Eq.}{Equations}
\crefname{figure}{Fig.}{Figs.}
\crefname{table}{Tab.}{Tabs.}
\crefname{section}{Sec.}{Sec.}
\crefname{theorem}{Thm.}{Thm.}
\crefname{lemma}{Lemma}{Lemmas}
\crefname{corollary}{Cor.}{Cor.}
\crefname{example}{Example}{Examples}
\crefname{remark}{Remark}{Remarks}
\crefname{algorithm}{Alg.}{Algorithms}
\crefname{appendix}{Appendix}{Appendices}
\crefname{subappendix}{Appendix}{Appendices}
\crefname{subsubappendix}{Appendix}{Appendices}

\newcommand{\spX}{\mathcal{X}}
\newcommand{\spA}{\mathcal{A}}

\newcommand{\trop}{\msf{T}}
\newcommand{\polop}{\msf{P}}
\newcommand{\pol}{\pi}

\newcommand{\Id}{\msf{Id}}

\title{Reward-free Pretraining for Reinforcement Learning \\via Occupancy Coverage Maximization}

\author{ 
   Marco Prattic\`o$^{1}$\\
  {\footnotesize\texttt{marco.prattico@iit.it}}
  \and 
  Pietro Novelli$^{1}$ \\ 
  {\footnotesize\texttt{pietro.novelli@iit.it}} 
  \and Massimiliano Pontil$^{1,2}$\\
  {\footnotesize\texttt{massimiliano.pontil@iit.it}}
  \and Carlo Ciliberto$^{2,3}$\\ 
  {\footnotesize\texttt{ c.ciliberto@ucl.ac.uk}}  
}
\date{}

\usepackage[textsize=tiny]{todonotes}

\begin{document}

\maketitle
\footnotetext[1]{Computational Statistics and Machine Learning - Italian Institute of Technology, Genoa, Italy}\footnotetext[2]{AI Centre, Computer Science Department, University College London, London, UK.}
\footnotetext[3]{The Italian Institute of Artificial Intelligence (AI4I), Turin, Italy .}
\setcounter{footnote}{3}

\begin{abstract}
\noindent Sparse rewards pose a central challenge in reinforcement learning, since agents receive no informative signal until they reach their goal. Intrinsic-reward methods address this issue by optimizing non-stationary objectives such as novelty, prediction error, or skill diversity, thereby injecting a supervision signal into the problem. While effective, these methods often require that the extrinsic (sparse) reward can be evaluated -- either online or during offline relabeling of the stored transitions. This limitation is particularly vexing for multi-task, meta-, and continual reinforcement learning, where agents' interactions with the environment are usually {\em reward-free}. In this work, we present a method to pre-train {\em transferable exploration policies} that rapidly adapt to sparse rewards at downstream task time. Our objective maximizes state-space covering for the occupancy measure, and can be framed in terms of entropy maximization. Its algorithmic implementation, ROVER \footnote{Code and hyperparameters available at: \href{https://github.com/marcopra/ROVER.git}{\texttt{github.com/marcopra/ROVER}}.}, leverages recent advances on the operatorial formulation of RL to estimate occupancy with a learned resolvent world model, bypassing common hurdles associated with density and entropy estimation. ROVER further introduces a virtual ``sink" state for unexplored regions, balancing coverage of known states with expansion into unseen ones and preventing cyclic expansion–collapse behavior during learning. In simple sparse navigation tasks, ROVER produces more uniform aggregate coverage and stronger initializations for downstream tasks than standard reward-free baselines.

\end{abstract}

\section{Introduction}
Sparse rewards remain a central obstacle for reinforcement learning (RL) \cite{sutton1998reinforcement}.
As informative feedback is observed only after reaching a task-dependent event, the behavior available at the beginning of training becomes of primary importance for downstream learning.
Reward-free pretraining addresses this observation by letting the agent interact with the environment before the reward is specified, and can precondition the explicit-reward problem through data \cite{yarats2022dontchangealgorithmchange}, representations \cite{touati2021learning, agarwal2024proto, tirinzoni2025zero, barreto2017successor}, world models \cite{hansen2023tdmpc2, hafner2023mastering}, or policies \cite{laskin2021urlb}.
However, data- and representation-based transfer often assumes that the downstream reward function can be evaluated offline after pretraining, for instance by relabeling stored transitions.
This assumption is often inappropriate in multi-task, meta, or continual sparse-reward problems, where meaningful feedback is observed only online, when the agent actually reaches the task-dependent goal.

Many reward-free objectives are motivated by coverage.
Maximum-entropy and state-marginal objectives seek broad state occupancies \cite{hazan2019maxent, mutti2020policy, yarats2021reinforcement, lee2019smm, liu2021apt}; prediction-error and curiosity methods reward novelty relative to a learned model \cite{pathak2017curiosity, burda2018rnd, pathak2019self}; and skill-diversity methods encourage distinguishable outcomes \cite{tolguenec2024explorationlearningdiverseskills, laskin2022CIC}.
In practice, these objectives are often optimized through non-stationary proxies that change with the data, the density estimate, the representation, or the auxiliary model.
As a result, a method may discover many states during pretraining and still end with a policy whose discounted occupancy is concentrated in a small region \cite{lee2019smm, burda2018rnd}.

In this work, we focus on policy transfer when the downstream reward is unavailable both during pretraining and at any later offline relabeling stage.
We formulate reward-free policy pretraining as \emph{occupancy coverage}: learning a {\em stationary} Markov policy whose discounted state occupancy is broadly distributed over the reachable state space.
Concretely, we embed the state occupancy measure in a Reproducing Kernel Hilbert Space (RKHS)~\cite{muandet2017kernel} and minimize its squared Hilbert norm, which penalizes concentrated occupancies under the kernel-induced notion of similarity.
This objective is equivalent to maximizing the metric R\'enyi entropy of order $2$~\cite{leinster2021maximum}; in the tabular setting, it is also equivalent to matching the uniform measure in MMD distance, see~\cref{prop:tabular_equivalence}.
With learned representations, the same objective promotes diversity in the similarity geometry induced by the representation.

We implement this objective in \ALGO.
ROVER is an end-to-end algorithm that learns a representation, estimates the kernel mean embedding of the occupancy measure in closed form, and computes policy gradients for its squared norm.
We solve the resulting policy-improvement step by policy mirror descent in closed form.
A key component of ROVER is a virtual ``sink'' state that augments the learned transition operator to handle unsupported state-action regions.
This augmentation yields an explicit cost for assigning occupancy outside the current data support and prevents the expansion--collapse behavior observed with purely data-supported operators.

Our contributions are threefold.
First, we formalize occupancy coverage as a target-free RKHS objective for sparse-reward policy transfer and relate it to MMD matching, R\'enyi-2 diversity, and potential-theoretic energy minimization.
Second, we derive an operator-based policy update for minimizing this objective from a learned transition model.
Third, we introduce a sink-state augmentation for the transition operator and show empirically that it stabilizes coverage under partial data.
Experiments in tabular and pixel-based sparse navigation environments show that \algo{} induces more transferable coverage behavior than standard reward-free baselines.

\section{Problem Statement: Occupancy Coverage}\label{sec:problem_statement}

Consider a discounted Markov decision process with state space $\spX$, action space $\spA$, transition kernel $\tau(\cdot\mid x,a)$, initial state distribution $\nu_0$, and discount factor $\gamma\in(0,1)$.
For a stationary Markov policy $\pi(\cdot\mid x)$, its discounted state occupancy is the probability measure defined by
\begin{equation}\label{eq:discounted_occupancy}
d_{\nu_0}^{\pi}(B)
= (1-\gamma)\sum_{t=0}^{\infty}\gamma^t
\mathbb P_\pi(x_t\in B\mid x_0\sim \nu_0),
\qquad B\subseteq\spX\ \text{measurable},
\end{equation}
where the probability is taken over trajectories induced by $\pi$ and $\tau$.
Equivalently, $d_{\nu_0}^{\pi}(B)$ captures the probability of stopping in $B$ when starting from $\nu_0$ and then following $\pi$ for a number of steps sampled from the geometric distribution with mean $\gamma/(1-\gamma)$.
We write $d^\pi$ when $\nu_0$ is fixed.

Exploration strategies for sparse-reward problems often try, explicitly or implicitly, to spread the state occupancy.
The standard formulation of this principle is to seek a policy whose state occupancy has maximal entropy.
This viewpoint underlies maximum-entropy exploration methods \cite{hazan2019maxent, mutti2020policy, yarats2021reinforcement} and is also the reference point for state-marginal matching, which casts coverage as minimizing $D(d^\pi,u)$ for a reference state distribution $u$ and a divergence or discrepancy $D$ (such as the KL) \cite{lee2019smm}.
The closest match to our work is MEPOL \cite{mutti2020policy}. However, it relies on a hand-specified state metric, making nontrivial the extension to high-dimensional states, such as raw images, motivating ROVER’s learned latent-kernel formulation.
Other methods implement the same objective less directly: APT estimates a particle-based entropy surrogate \cite{liu2021apt}; RND and related curiosity methods use prediction error as a time-varying novelty signal \cite{burda2018rnd,pathak2017curiosity,pathak2019self}; and skill-diversity methods encourage coverage through distinguishable latent-conditioned outcomes \cite{laskin2022CIC,tolguenec2024explorationlearningdiverseskills}.

The practical difficulty is that many implementations couple the coverage criterion to a moving auxiliary object, such as a predictor, density estimator, nearest-neighbor geometry, or skill discriminator.
Such non-stationarity can produce policies that transiently discover new regions without retaining them in the final discounted occupancy.
The stronger requirement of covering states without forgetting them along an individual trajectory is also generally history-dependent: deciding whether the next state improves coverage depends on the set of states already visited \cite{mutti2022importanceofnonmarkovianity}.
Our objective instead optimizes the expected discounted occupancy of a stationary Markov policy, seeking broad coverage across trajectories.

We use kernel mean embeddings to express coverage as a discrepancy between occupancies \cite{muandet2017kernel}.
Let $\phi:\spX\to\mathcal H$ be a feature map into an RKHS with kernel $k(x,y)=\langle \phi(x),\phi(y)\rangle_{\mathcal H}$.
For probability measures $\rho$ and $q$ over $\spX$, define
\begin{equation}
\mu_\rho
= \int_{\spX}\phi(x)\,\rho(dx) \qquad \textrm{and} \qquad
{\rm MMD}^2(\rho,q)
= \|\mu_\rho-\mu_q\|_{\mathcal H}^2. \label{eq:mmd_objective}
\end{equation}
If a target distribution $u$ were known, minimizing ${\rm MMD}^2(d^\pi,u)$ would be a direct occupancy-matching objective.
However, a uniform target may be undefined on non-compact spaces, expensive to estimate in high-dimensional observation spaces, and infeasible when the reachable occupancy set does not contain the uniform distribution.
The following finite-state calculation shows that, in the tabular case, the target can be removed without altering the minimizer.

\begin{proposition}\label{prop:tabular_equivalence}
Let $\spX$ be finite with $|\spX|=N$, let $k(x,y)=\delta_{xy}$, and let $u$ be the uniform distribution over $\spX$.
Then, for any policy $\pol$ we have $\|\mu_{d^\pi}\|_{\mathcal H}^2=\sum_{x\in\spX}d^\pi(x)^2$ and
\begin{equation}
    \arg\min_{\pi\in\Pi}~~ {\rm MMD}^2(d^\pi,u)
    =
    \arg\min_{\pi\in\Pi}~~ \|\mu_{d^\pi}\|_{\mathcal H}^2.
\end{equation}
\end{proposition}

\begin{proof}
With the Dirac kernel, the embedding is the probability vector itself; hence
\begin{equation}
\|\mu_{d^\pi}\|_{\mathcal H}^2
= \sum_{x,y\in\spX} d^\pi(x)d^\pi(y)\delta_{xy}
= \sum_{x\in\spX}d^\pi(x)^2.
\end{equation}
Furthermore,
\begin{align}
{\rm MMD}^2(d^\pi,u)
&= \sum_{x\in\spX}\left(d^\pi(x)-\frac{1}{N}\right)^2 \nonumber = \sum_{x\in\spX}d^\pi(x)^2 - \frac{2}{N}\sum_{x\in\spX}d^\pi(x) + \frac{1}{N}
 = \|\mu_{d^\pi}\|_{\mathcal H}^2 - \frac{1}{N},
\end{align}
because $\sum_x d^\pi(x)=1$.
The two objectives, therefore, differ by a constant independent of $\pi$.
\end{proof}

This motivates the target-free self-similarity objective
\begin{equation}\label{eq:objective}
    \mathcal J(\pi)
    = \|\mu_{d^\pi}\|_{\mathcal H}^2
    = \iint k(x,y)\,d^\pi(dx)\,d^\pi(dy).
\end{equation}
The objective penalizes assigning probability mass to pairs of states that are similar under the kernel. For normalized nonnegative similarity kernels, \cref{eq:objective} is the inverse of the order-2 diversity of $d^\pi$ \cite{leinster2021maximum}; with the Dirac kernel this reduces to the inverse collision probability, and $\log(1/\mathcal J(\pi))$ is the R\'enyi entropy of order $2$.
\Cref{app:Reniy_diversity} expands this interpretation, while \cref{app:potential_theory} gives the complementary potential-theoretic view of \cref{eq:objective} as an energy minimization problem.
In learned representation spaces, \cref{eq:objective} should not be read as uniformity over raw observations, rather as uniformity over representations.
\algo{} makes such geometry explicit, estimates the occupancy induced by the current representation through an operator model, and controls unsupported regions through the sink-state mechanism introduced in \cref{sec:sink_state}.

\section{Proposed Algorithm: \algo}\label{sec:rover}
\algo{} uses the operator formulation of Markov decision processes to make the occupancy objective computable.
Following~\cite{novelli2024operator} and recent work on learning Koopman-style transition operators in controlled systems \cite{takeishi2017learning,lusch2018deep,otto2019linearly,kostic2022learning,kostic2023sharp,nuske2023finite,jeong2025efficient,turri2025self,rozwood2024koopman}, we represent state-distribution evolution through linear operators.
Once such an operator is estimated from data, discounted occupancies of candidate policies can be written and differentiated in closed-form.
For a state-action distribution $\rho\in\mathcal M(\spX\times\spA)$ and a state distribution $\nu\in\mathcal M(\spX)$, define the transition and policy operators by
\begin{equation}\label{eq:operators}
    (\trop\mu)(B)
    = \int_{\spX\times\spA}\tau(B\mid x,a)\,\rho(dx,da),
    \qquad
    (\polop_\pi\nu)(C)
    = \int_{\spX}\int_{\spA}\mathbf 1_C(x,a)\,\pi(da\mid x)\,\nu(dx),
\end{equation}
for measurable $B\subseteq\spX$ and $C\subseteq\spX\times\spA$.
Then the state distribution evolves as $\nu_{t+1}=\trop\polop_\pi\nu_t$, and the discounted occupancy admits the resolvent form
\begin{equation}\label{eq:state_occupancy}
d_{\nu_0}^{\pi}
= (1-\gamma)\sum_{t=0}^{\infty}\gamma^t(\trop\polop_\pi)^t\nu_0
= (1-\gamma)(\Id-\gamma\trop\polop_\pi)^{-1}\nu_0.
\end{equation}
Since our objective $\mathcal{J}(\pol)$ in \cref{eq:objective} involved the MMD, we consider the natural lifting of $\trop$ and $\polop$ to mean embeddings. More precisely, given two representations $\psi:\X\times\A\to\mathcal{G}$ and $\phi:\X\to\mathcal{H}$ into two feature spaces $\mathcal{G}$ and $\mathcal{H}$ respectively, we can define $\tilde\trop \mu^\psi_\rho = \mu^\phi_{\trop\rho}$, with $\mu^\psi$ and $\mu^\phi$ denoting the corresponding mean embeddings. Where clear from the context, with some abuse of notation, we will use the simplified notation $\trop \mu_{\rho} = \mu_{\trop\rho}$. The policy operator can be lifted in the same way, so the resolvent characterization in \cref{eq:state_occupancy} extends to mean embeddings and can be readily used to reformulate our objective as
\begin{equation}
    \mathcal{J}(\pol) = \|(1-\gamma)(\Id - \gamma \trop \polop_\pol)^{-1}\mu_{\nu_0}\|_\hh^2
\end{equation} 

Leveraging the formulation above, ROVER estimates the lifted version of $\trop$ and uses the corresponding resolvent to differentiate an approximate version of $\mathcal J(\cdot)$. The resulting method alternates between data collection and three steps: representation learning, transition-operator estimation, and policy mirror descent, as summarized in \cref{alg:rover_main}.

\begin{algorithm}[t]
\caption{\algo{} pretraining loop}
\label{alg:rover_main}
\begin{algorithmic}[1]
\REQUIRE Discount $\gamma$, PMD step size $\eta$, KRR regularization $\lambda$, sink scale $\eps$
\STATE Initialize behavior policy $p=p_0$ and replay buffer $\mathcal D=\emptyset$
\FOR{each pretraining round}
    \STATE Collect transitions using $p$ and add them to $\mathcal D$
    \STATE \textbf{Representation learning:} update $\phi_\theta$ on $\mathcal D$ using \cref{eq:l_nce}
    \STATE \textbf{World-model estimation:} compute $\bar\trop_n$ by the sink-biased regression in \cref{eq:augmented_cme_opt}
    \STATE \textbf{Policy mirror descent:} initialize $\pi_0(\cdot\mid x)=1/|\spA|$ and $C=0$
    \FOR{$j=0,\ldots,T_{\mathrm{PMD}}-1$}
        \STATE Compute $c_{\pi_j}$ from \cref{thm:dual_gradient}
        \STATE $C\gets C+c_{\pi_j}$
        \STATE $\pi_{j+1}(\cdot\mid x)
        \gets
        \sigma\!\left(
        -\eta\langle\psi(x,\cdot),\Psi_n^\top C\rangle\right)$
    \ENDFOR
    \STATE Update $p\gets\pi_{T_{\mathrm{PMD}}}$
\ENDFOR
\STATE \textbf{return} pretrained policy $p$ and representation $\phi_\theta$
\end{algorithmic}
\end{algorithm}

\paragraph{Step 1: Representation learning}
The feature map $\phi$ determines the geometry of the coverage objective, so ROVER learns $\phi$ from the transition data rather than fixing it a priori.
Given transitions $(x,a,x')$, we compute $L_1$-normalized latent states $\phi_\theta(x)=f_\theta(x)/\|f_\theta(x)\|_1$ and state-action features $\psi_\theta(x,a)=\phi_\theta(x)\otimes e_a$, where $e_a$ is the one-hot vector for action $a$.
A linear predictor $W$ is trained to align the predicted next-state representation $W\psi_\theta(x,a)$ with $\phi_\theta(x')$ using an InfoNCE loss~\cite{oord2018representation,laskin2020curl,zheng2023taco}:
\begin{equation}\label{eq:l_nce}
    \mathcal{L}_{\text{NCE}}(\phi_{\theta}, W)
    = -\log \frac{\exp(\langle \bar{\phi}_{\theta}(x'_{i}), W\bar{\psi}_{\theta}(x_{i}, a_{i}) \rangle)}
    {\sum_{j} \exp(\langle \bar{\phi}_{\theta}(x'_{j}), W\bar{\psi}_{\theta}(x_{i}, a_{i}) \rangle)},
\end{equation}
where $\bar v=v/\|v\|_2$ and negatives are drawn from the batch.
This contrastive loss is an implementation choice rather than a structural requirement of the method.
Other self-supervised objectives for predictive or bootstrap representation learning, such as BYOL-style losses \cite{guo2022byol} or self-predictive auxiliary losses \cite{pathak2019self}, could be used in the same role.
For the policy update, the learned representation supplies both the kernel in \cref{eq:objective} and the state-action features used by the operator estimate.

\paragraph{Step 2: Transition-operator estimation}
The estimation problem is to approximate the lifted transition operator from sampled transitions.
For a Dirac measure at $(x,a)$, $\trop\delta_{(x,a)}=\tau(\cdot\mid x,a)$.
Since $\mu^\psi_{\delta_{(x,a)}}=\psi(x,a)$, the lifted operator should map $\psi(x,a)$ to $\mu^\phi_{\tau(\cdot\mid x,a)}$, that is, to the conditional mean $\mathbb E[\phi(x')\mid x,a]$.
In the implementation we use $\psi(x,a)=\phi(x)\otimes e_a$, where $e_a$ is the one-hot encoding of the discrete action.
A conditional mean embedding estimates this operator as a linear map from state-action features to next-state embeddings \cite{grunewalder2012modelling,muandet2017kernel}.
Given a dataset $\mathcal D=\{(x_i,a_i,x'_i)\}_{i=1}^n$ and the learned feature map, ROVER estimates
$\mathbb E[\phi(x')\mid x,a]\approx \trop\psi(x,a)$.
The estimator is the regularized least-squares solution of the vector-valued regression problem
\begin{equation}\label{eq:KRR}
    \min_{\trop} ~~ \sum_{i=1}^n
    \|\phi(x'_i)-\trop\psi(x_i,a_i)\|_{\mathcal H}^2
    + \lambda\|\trop\|_{\mathrm{HS}}^2,
\end{equation}
namely
\begin{equation}\label{eq:cme_estimator}
    \hat\trop_n = \Phi_n^\top(K+\lambda \Id)^{-1}\Psi_n.
\end{equation}
Here $\Phi_n$ has ``rows'' $\phi(x'_i)^\top$, $\Psi_n$ has ``rows'' $\psi(x_i,a_i)^\top$, and $K=\Psi_n\Psi_n^\top$ is the state-action Gram matrix.
This is the standard kernel ridge regression estimator of the conditional mean embedding.
Its finite-dimensional sandwich form is what makes the policy update below computable.


\paragraph{Step 3: Policy mirror descent}
Given an estimated transition operator $\hat \trop_n$, let $R_\pi=(\Id-\gamma\hat\trop_n\polop_\pi)^{-1}$.
Differentiating $\mathcal J(\pi)$ through the resolvent gives the functional gradient
\begin{equation}\label{eq:primal_gradient}
\nabla_{\polop}\mathcal J(\pi)
= 2\gamma(1-\gamma)^2
\hat\trop_n^*R_\pi^*R_\pi\mu_{\nu_0}\otimes R_\pi\mu_{\nu_0},
\end{equation}
which can be employed in first-order optimization methods to minimize the objective function~\cref{eq:objective}. Concretely, the gradient can be evaluated thanks to the kernel trick, as shown in the following theorem, proved in \cref{app:dual_gradient_proof}.
\begin{theorem}[Dual Form of the Gradient]\label{thm:dual_gradient}
Consider a transition operator with ``sandwich'' structure
\begin{equation}\label{eq:sandwich_form}
\hat{\trop}=\Phi^\top G\Psi,
\end{equation}
where $\Phi\in\mathbb R^{n\times d_\phi}$ has rows $\phi(x'_i)^\top$, $\Psi\in\mathbb R^{n\times d_\psi}$ has rows $\psi(x_i,a_i)^\top$, and $G\in\mathbb R^{n\times n}$ is symmetric.
Define $K_\phi=\Phi\Phi^\top$ and $(M_\pi)_{ij}
= \sum_a \pi(a\mid x'_i)\langle \psi(x'_i,a),\psi(x_j,a_j)\rangle$.
If $\mu_0=\Phi^\top\alpha_0$, a direct substitution of~\cref{eq:sandwich_form} into~\cref{eq:primal_gradient} gives
\begin{equation}\label{eq:dual_gradient}
[\nabla_\pi\mathcal J(\pi)](x,a)
= \langle \psi(x,a),\Psi^\top c_\pi\rangle,
\end{equation}
where $c_\pi = 2\gamma(1-\gamma)^2 G(\Id-\gamma M_\pi^\top G)^{-1} K_\phi (\Id-\gamma GM_\pi)^{-1}\alpha_0.$
\end{theorem}

Similarly to~\cite{shani2020adaptive,agarwal2021,xiao2022}, ROVER optimizes its objective via Policy Mirror Descent, which guarantees that at each iteration the updated $\pi$ is a valid conditional probability distribution. Since the KRR estimator in \cref{eq:cme_estimator} satisfies \cref{eq:sandwich_form} with $G=(K+\lambda \Id)^{-1}$, the Policy Mirror Descent iterates can be written, leveraging \cref{thm:dual_gradient}, as
\begin{equation}\label{eq:pmd_simplified}
\pi_{t+1}(\cdot\mid x)
= \sigma\left(
\log\pi_0(\cdot\mid x)
-\eta\Big\langle \psi(x,\cdot),
\Psi_n^\top\sum_{s=0}^{t}c_{\pi_s}
\Big\rangle
\right),
\end{equation}
where $\sigma$ is the softmax.
We note that the optimization history is compressed into the vector $\sum_{s=0}^{t}c_{\pi_s}$ of accumulated coefficients.

\subsection{Sink-State Augmentation}\label{sec:sink_state}

The estimator $\hat\trop_n$ in \cref{eq:cme_estimator} is supported by the state-action features observed in the dataset.
For a query $(x,a)$ far from this support, the kernel vector $\Psi_n\psi(x,a)$ is close to zero, and therefore $\hat\trop_n\psi(x,a)\approx 0$.
Since the objective minimizes an RKHS norm, this zero prediction can be incorrectly treated as a low-cost continuation of the dynamics.
In fact, it only indicates that the learned operator has no data with which to predict the next-state embedding.

\Cref{fig:sink_multiroom} (Top row) illustrates the resulting instability.
The diagnostic uses a zero-memory buffer, so the operator is fit only on the most recent batch.
Without an explicit representation of unsupported regions, previously visited rooms can disappear from the support of the next operator fit, and the policy oscillates between expansion and rediscovery.
With a replay buffer this effect is less abrupt, but it will occur eventually through finite capacity, imbalanced sampling, or representation drift.

\begin{figure}[t]
    \centering
    \includegraphics[width=1\linewidth]{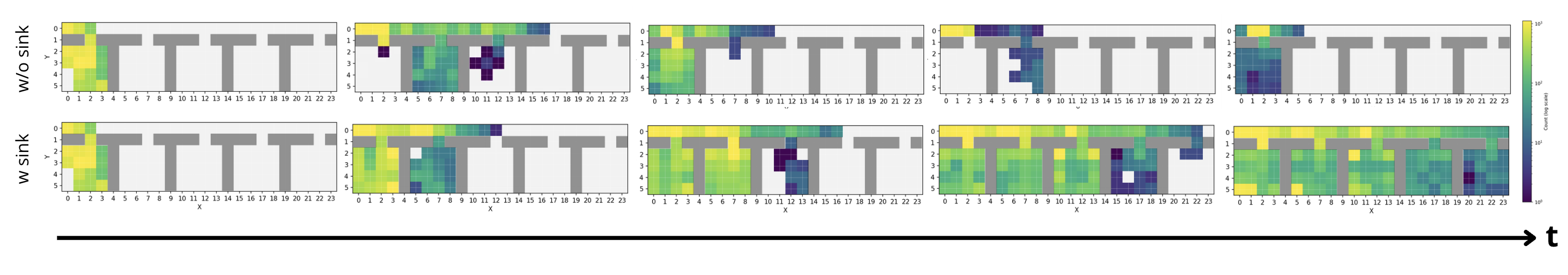}
    \caption{Effect of sink-state augmentation under an intentionally severe zero-buffer update in the multi-room environment. Each column shows a later policy-optimization round, with visitation counts plotted on a log scale. \textbf{Top: without the sink state}, the learned operator is fit only to the most recent batch, so previously covered regions become unsupported and the policy oscillates over the first rooms. \textbf{Bottom: with the sink state}, unsupported mass is routed to a controlled absorbing state, yielding policies that expand while retaining connected coverage of earlier regions.}
    \label{fig:sink_multiroom}
\end{figure}

We therefore augment the state space with a virtual absorbing state $s_{\sink}$ representing the part of the model unsupported by data.
Let $\bar{\spX}=\spX\cup\{s_{\sink}\}$ and define $\bar\phi:\bar{\spX}\to\bar{\mathcal H}=\mathcal H\oplus\mathbb R$ as
\begin{equation}
\bar{\phi}(x) = \begin{cases}
(\phi(x),0) & \text{if } x\in\spX,\\
(0_{\mathcal H},\eps) & \text{if } x=s_{\sink},
\end{cases}
\end{equation}
where $\eps>0$ fixes the norm of one unit of sink occupancy.
We use the absorbing prior $\trop_0=\bar\phi_{\sink}\otimes\mathbf 1$, under the normalization $\langle\mathbf 1,\bar\psi(x,a)\rangle=1$, and estimate the transition operator by the shifted ridge problem
\begin{equation}\label{eq:augmented_cme_opt}
\bar\trop_n
= \trop_0
+ \bigl(\bar\Phi_n^\top-\bar\phi_{\sink}\mathbf 1_n^\top\bigr)
(K+\lambda \Id)^{-1}\bar\Psi_n.
\end{equation}
Here $\bar\Phi_n$ and $\bar\Psi_n$ contain the observed augmented next-state and state-action embeddings, and
$K=\bar\Psi_n\bar\Psi_n^\top$.

\begin{proposition}[Prediction of the augmented estimator]\label{prop:augmented_prediction}
Let $\mathbf k(x,a)=\bar\Psi_n\bar\psi(x,a)$ and
$\boldsymbol\alpha(x,a)=(K+\lambda \Id)^{-1}\mathbf k(x,a)$.
Then the estimator in \cref{eq:augmented_cme_opt} satisfies
\begin{equation}
\bar\trop_n\bar\psi(x,a)
= \bar\Phi_n^\top\boldsymbol\alpha(x,a)
+ \bigl(1-\mathbf 1_n^\top\boldsymbol\alpha(x,a)\bigr)\bar\phi_{\sink}.
\end{equation}
In particular, if $\mathbf k(x,a)\approx 0$, then
$\bar\trop_n\bar\psi(x,a)\approx \bar\phi_{\sink}$.
\end{proposition}

\begin{proof}
Evaluating \cref{eq:augmented_cme_opt} at $\bar\psi(x,a)$ gives
\[
\bar\trop_n\bar\psi(x,a)
=\trop_0\bar\psi(x,a)
+(\bar\Phi_n^\top-\bar\phi_{\sink}\mathbf 1_n^\top)
(K+\lambda \Id)^{-1}\mathbf k(x,a).
\]
Since $\trop_0=\bar\phi_{\sink}\otimes\mathbf 1$ and
$\langle\mathbf 1,\bar\psi(x,a)\rangle=1$, the first term is
$\bar\phi_{\sink}$.
Collecting terms gives the claim.
\end{proof}

The coefficient $1-\mathbf 1_n^\top\boldsymbol\alpha(x,a)$ is the part of the prediction not explained by the training inputs.
The augmented estimator routes this residual to $s_{\sink}$ rather than to the zero vector.
The bottom row of \cref{fig:sink_multiroom} shows the corresponding effect on the learned policy sequence.

\begin{proposition}[Finite-dimensional structure]\label{prop:augmented_sandwich_main}
Under the normalization above, define
\[
    \widetilde\Phi_n^\top
    =
    \begin{bmatrix}
    \bar\Phi_n^\top-\bar\phi_{\sink}\mathbf 1_n^\top & \bar\phi_{\sink}
    \end{bmatrix},
    \qquad
    \widetilde\Psi_n
    =
    \begin{bmatrix}
    \bar\Psi_n\\
    \mathbf 1^\top
    \end{bmatrix},
    \qquad
    \widetilde G
    =
    \mathrm{diag}((K+\lambda \Id)^{-1},1).
\]
Then the augmented estimator admits the sandwich form
\begin{equation}
    \bar\trop_n=\widetilde\Phi_n^\top\widetilde G\widetilde\Psi_n.
\end{equation}
Moreover, the policy-induced matrix associated with the augmented operator is $\widetilde M_\pi=\mathrm{diag}(M_\pi,1)$, where $M_\pi$ is defined as in~\Cref{thm:dual_gradient}.
Consequently, the dual gradient formula in \cref{thm:dual_gradient} applies to the augmented model after replacing $(\Phi,\Psi,G,M_\pi)$ by $(\widetilde\Phi,\widetilde\Psi,\widetilde G,\widetilde M_\pi)$.
\end{proposition}

The detailed derivation of \cref{prop:augmented_sandwich_main} is given in \cref{prop:augmented_sandwich}.
The block-diagonal form leaves the PMD algebra unchanged but augments the learned model with an explicit unsupported component.

\begin{theorem}[Sink-state objective decomposition]\label{thm:sink_resolvent_decomp}
Let $\bar\mu_0=(\mu_0,0)$ be the initial embedding with no sink mass, and let
$\bar\polop_\pi$ be an extension of $\polop_\pi$ to $\bar{\spX}$ with the absorbing sink block described above.
Assume that $I-\gamma\bar\trop_n\bar\polop_\pi$ is invertible, and let
\[
    \bar\mu(\pi)
    =
    (1-\gamma)(I-\gamma\bar\trop_n\bar\polop_\pi)^{-1}\bar\mu_0
    \in\bar{\mathcal H}
\]
be the augmented discounted occupancy embedding.
Let $P_{\mathcal H}$ be the orthogonal projection from
$\bar{\mathcal H}=\mathcal H\oplus\mathbb R$ onto $\mathcal H$, and define
\[
    \hat\mu(\pi)=P_{\mathcal H}\bar\mu(\pi),
    \qquad
    m_{\sink}(\pi)
    =
    \frac{\langle \bar\mu(\pi),\bar\phi_{\sink}\rangle_{\bar{\mathcal H}}}
    {\|\bar\phi_{\sink}\|^2_{\bar{\mathcal H}}}.
\]
Then
\begin{equation}\label{eq:sink_decomposition}
    \|\bar\mu(\pi)\|^2_{\bar{\mathcal H}}
    =
    \|\hat\mu(\pi)\|^2_{\mathcal H}
    + \eps^2 m_{\sink}(\pi)^2.
\end{equation}
\end{theorem}

\begin{proof}
By construction $\bar\phi_{\sink}\in\{0_{\mathcal H}\}\oplus\mathbb R$ and
$\|\bar\phi_{\sink}\|^2_{\bar{\mathcal H}}=\eps^2$.
The decomposition
$\bar\mu(\pi)=\hat\mu(\pi)+m_{\sink}(\pi)\bar\phi_{\sink}$
is orthogonal, hence the Pythagorean identity gives \cref{eq:sink_decomposition}.
\end{proof}

Thus $\eps$ controls the relative cost of assigning occupancy to unsupported regions.
Small values make unsupported continuations cheap; large values make the policy optimization conservative.
During pretraining, the sink term is used only to regularize the learned operator and policy update.
After pretraining, the transferred object is the learned representation and policy.

To avoid the $\mathcal O(n^3)$ cost of full kernel inversion, the implementation uses a Nystrom approximation \cite{williams2000nystrom,rudi2015less} with $m\ll n$ anchor points from the replay buffer, reducing the dominant matrix inversion cost to $\mathcal O(m^3)$ per update.
A simplified overview and the matrix-level implementation details are given in \cref{alg:rover_simple,alg:dist_matching}.

\section{Experiments}

Since the primary contribution of this work is the derivation of a stable, closed-form exploration objective~\cref{eq:objective}, our experimental section was designed as a controlled diagnostic probe rather than a benchmark of raw capability. 
We engineered these environments from scratch to isolate specific topological pathologies, as bottlenecks or hard-to-explore subspaces, that induce the oscillatory failures affecting standard reward-free methods. This granular control allows us to directly visualize the collapse and expansion of the occupancy measure $\mu(\pi)$ in the RKHS, which is impossible in high-dimensional latent spaces. 
Further, such a setting helped to verify the role of the introduced augmented state by observing the mass transfer to the sink state $s_e$, and demonstrate the stability in settings where the non-augmented density-based approach is shown to diverge (\cref{fig:sink_multiroom}). 

We evaluate our method's ability to efficiently explore and keep a covering behaviour on a suite of fully observable, hard-to-explore MiniGrid-style environments \cite{chevalier2023minigrid} (\cref{app:envs-info}). Here, we consider mainly three configurations: the \textit{Middle-Room} \cref{fig:middle-room}, a room in the middle connected with corridors to rooms on each side, \textit{Multiple Rooms}, a series of 5 rooms connected by a corridor, and the \textit{Maze} configuration \cref{fig:maze_examples}, while other environmental architectures are considered in the appendix. We illustrate a snapshot of these environments in \cref{app:envs-info}.

All the configurations offer a tabular or a pixel-based observation, with discrete actions ($|\spA|=4$), and the downstream task employs a sparse reward structure ($-1$ per step until the goal) to encourage shortest-path discovery. In the end, we prove the effectiveness of our method beyond gridworlds by considering the Pointmaze setting \cite{fu2020d4rl}, with discretized actions.  

\subsection{Behavior Induced by Reward-Free Objectives} \label{sec:induced-behaviour}

We analyze the behavior induced by different reward-free pretraining objectives in the \textit{Middle Room} environment. 
The observation is a one-hot encoding of the discrete state, although \algo{} still learns a representation using \cref{eq:l_nce}. 
We compare \algo{} against Random Network Distillation (RND)~\citep{burda2018rnd}, State Marginal Matching (SMM)~\citep{lee2019smm}, Unsupervised Active Pre-Training (APT)~\citep{liu2021apt}, MAXENT~\citep{hazan2019maxent}, and Contrastive Intrinsic Control (CIC)~\citep{laskin2022CIC}.

In \cref{fig:behaviour_plus_replay_buffer}, we visualize each method using $50$ sampled trajectories from a representative checkpoint, selected either near the point where full state-space coverage is achieved or near the end of pretraining. 
The snapshots highlight the difference between \emph{discovering} states during pretraining and obtaining a final policy whose occupancy broadly covers the state space. Several baselines visit different regions over the course of training, but their policy snapshots often concentrate on one region at a time (\cref{fig:induced_behaviour}). 
Thus, the replay buffer may be diverse, while the final transferred behavior is not.

\begin{figure}
    \centering
    \includegraphics[width=1.02\linewidth]{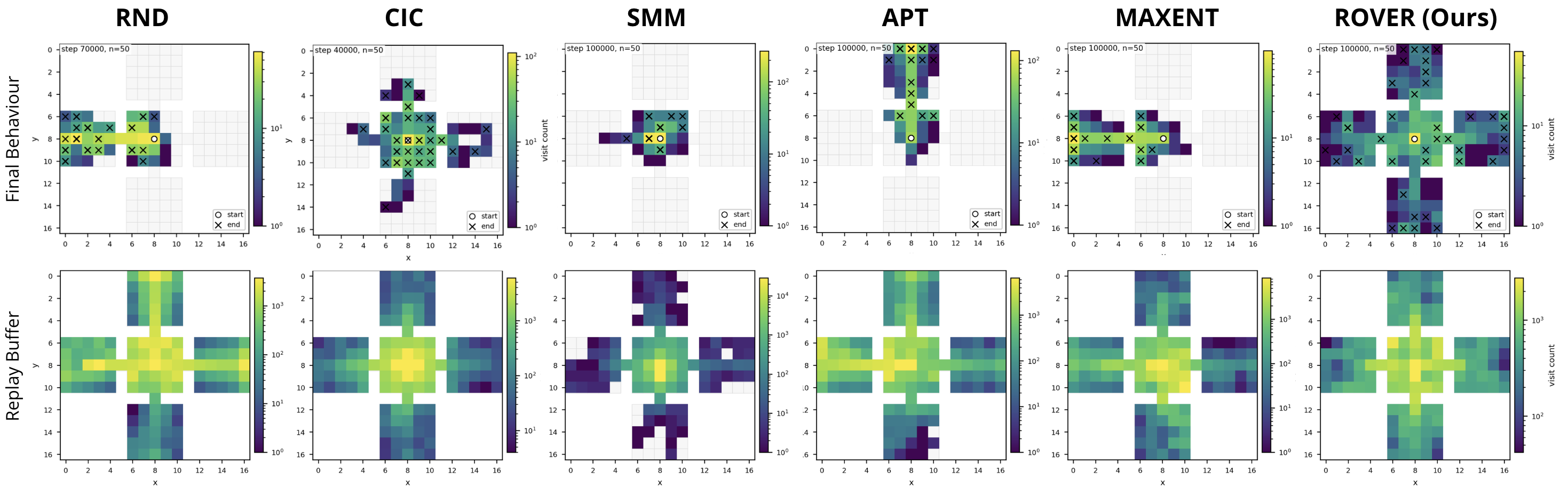}
    \caption{\textbf{Top: Behaviour of the resulting policy} For each method, we sample $50$ trajectories from a representative checkpoint during pretraining, selected either near full feasible state-space coverage or near the end of the pretraining window. \textbf{Bottom: Samples collected during pretraining}, we visualize the entire dataset collected by each method.}
    \label{fig:behaviour_plus_replay_buffer}
\end{figure}

This distinction matters in sparse-reward transfer. 
If the downstream reward can be evaluated offline, pretraining data can in principle be relabeled and reused. 
However, in many sparse-reward tasks, success is an interaction-time event observed only when the agent actually reaches the task-dependent goal. 

Oscillatory behavior in intrinsic-reward methods has been observed for methods such as SMM and RND by~\cite{lee2019smm}, and often stems from the non-stationarity of the exploration signal. In prediction-error methods such as RND~\cite{burda2018rnd}, unvisited states are rewarding until they become predictable, after which the policy is pushed toward other regions. Previously visited regions can later become attractive again due to function-approximation error, forgetting, or distribution shift, producing cycles of expansion and rediscovery.

Entropy-, density-, and marginal-matching methods such as APT, MAXENT, and SMM can exhibit a related effect. They encourage visitation of underrepresented states, but the set of underrepresented states changes as the replay distribution evolves. As a result, these methods may produce diversity in the data collection without ensuring that the final policy induces balanced aggregate coverage, especially in bottlenecked environments.

Skill-based methods, such as CIC, learn different behaviors by conditioning the policy on a latent skill. This can help the agent explore, but it does not directly force the final behavior to visit all reachable regions evenly. 
Even if we sample different skills at evaluation time, the resulting trajectories may still spend too much time in some regions and too little in others, especially when rooms are connected by narrow bottlenecks. Therefore, learning distinct skills can lead to diverse behaviors, but it does not necessarily produce the balanced coverage that ROVER explicitly optimizes.

Overall, these observations suggest that common reward-free objectives can be effective for discovery, but do not directly optimize the property needed for behavior transfer: a final policy whose discounted occupancy is broadly distributed over the reachable state space. 
By contrast, \algo{} directly targets occupancy coverage and progressively expands the visited state space while retaining previously discovered regions.

\subsection{Full Coverage in Controlled Settings}
\begin{figure}
    \centering
    \includegraphics[width=1\linewidth]{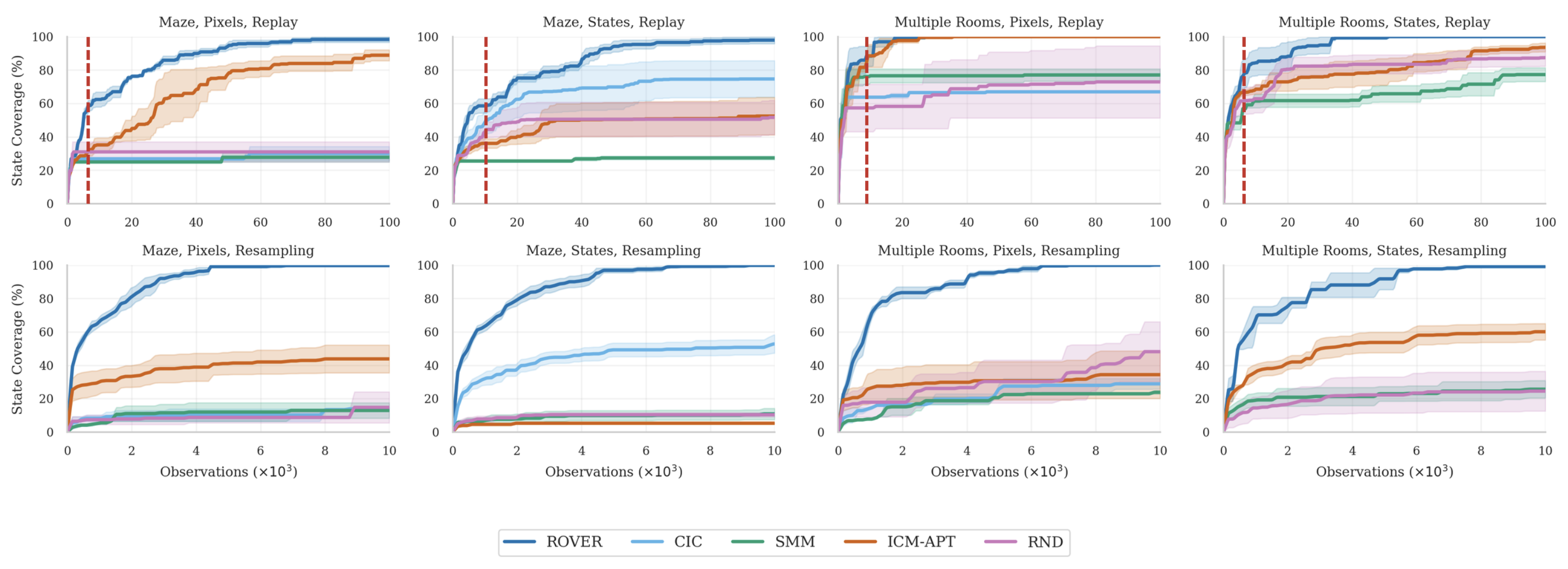}
    \caption{State-space coverage sample efficiency in \textit{Multi-Rooms} and \textit{Maze}. On the x-axis, the observations are obtained by sampling from the environment. 
\textbf{Top}: pretraining samples required to reach full coverage. We displayed values until $100 \times 10^3$. The red-dotted line shows when the full coverage is reached by resampling with ROVER's learned policy.
\textbf{Bottom}: samples required to reach full coverage by resampling from the final policy. We displayed values until $10 \times 10^3$, an order of magnitude less than exploring from scratch. 
}
    \label{fig:resampliong}
\end{figure}
In this section, we compare the number of samples required by \algo{} and the baselines to achieve full state-space coverage during reward-free pretraining in the \textit{Multi-Rooms} and \textit{Maze} environments, considering both state-based and pixel-based observations (\cref{fig:resampliong}). The results show that \algo{} reaches full coverage substantially faster than the baselines when run from scratch.

The learned policy can also be used after pretraining to efficiently resample trajectories. In both environments, resampling from the final \algo{} policy reduces the number of samples required to cover the state space by approximately one order of magnitude in each environment and setting. Thus, \algo{} provides two complementary benefits: it improves exploration efficiency during pretraining, and it produces a final policy that remains effective for state-space coverage.

This property is particularly useful in multi-task settings. A single pretrained policy can be reused across different downstream tasks in the same environment, so the pretraining cost should be viewed as an upfront investment and amortized over the number of fine-tuning tasks. We provide a more detailed analysis in \cref{app:overhead}. Such exploratory coverage can also remain useful under stochastic rewards, where the agent must repeatedly sample the environment to infer the underlying reward distribution, as in multi-armed bandit problems.

Overall, these results show that \algo{} not only explores the state space efficiently during pretraining, but also preserves a covering behavior at the end of training. This is the central property targeted by our objective. While the main focus of this section is therefore on the exploration behavior induced by \algo{}, we also demonstrate its compatibility with downstream model-free methods in \cref{sec:tabular-exps,app:model-free_init}.

\subsection{Beyond Gridworld: Position Coverage in PointMaze}
\label{sec:pointmaze}

To evaluate ROVER beyond gridworlds, we consider a PointMaze U-Maze environment \cite{fu2020d4rl}, with continuous position and velocity as observations and discretized actions (up, down, left, and right). 
We train each policy for $250{,}000$ environment timesteps and evaluate the final checkpoint. 
In this experiment, our coverage target is feasible ball position rather than the full state: if $d^\pi$ is the discounted occupancy over states $s=(x,y,v_x,v_y)$, we measure its position marginal $q^\pi=g_{\#}d^\pi$ under the projection $g(s)=(x,y)$. 
Thus, two states with the same position but different velocities count as the same location for coverage, while velocity remains available to the policy and dynamics model.
We distinguish the kernel used to represent transitions from the kernel used to measure coverage. In the state-observation setting, the dynamics kernel uses the full state and the coverage kernel uses only $(x,y)$.
In the pixel setting, the dynamics model receives a four-frame stack, whereas the coverage encoder processes a single frame. 
In \cref{app:projected_occupancy_coverage}, we report the mathematical formulation of the projected occupancy objectives, while implementation details are listed in Appendix~\ref{app:pointmaze_implementation}.

We partition the walkable region into a fixed grid of cells and measure coverage as the percentage of visited cells by trajectories sampled from the stochastic policy. 
We evaluate from three starting positions (top-left, top-right, and the center of the vertical corridor) to avoid bias given to a favourable initial location. 
For each algorithm and position, we run three training seeds per position, yielding nine training runs per method. Each run is evaluated with three evaluation seeds of 50 trajectories each; we average the coverage over these seeds, then report the mean and standard deviation across available training runs; a visual example is provided in \cref{fig:pointmaze_coverage}. ROVER achieves the highest coverage in both pixel and raw state settings (\cref{tab:pointmaze_coverage}). These results show competitive performance of ROVER beyond gridworld domains. Scaling to broader domains will require efficient subsampling to reduce the computational cost of the resolvent, as well as more scalable representations. We leave these directions to future work.

\begin{table}[t]
    \centering
    \caption{U-Maze cell coverage (\%) by observation modality and algorithms. For learned methods, each cell reports mean $\pm$ standard deviation across nine training runs (three starting locations, three training seeds each). The coverage of each run is averaged across three evaluation seeds, each using 50 trajectories from the stochastic policy.}
    \label{tab:pointmaze_coverage}
    \begin{tabular}{lcc}
        \toprule
        & \multicolumn{1}{c}{Pixel observations}
        & \multicolumn{1}{c}{State observations} \\
        \cmidrule(lr){2-2} \cmidrule(lr){3-3}
        Method & Mean $\pm$ std & Mean $\pm$ std \\
        \midrule
        Random Policy & $75.89 \pm 12.18$ & $75.89 \pm 12.18$ \\
        CIC           & $20.46 \pm 10.58$ & $81.99 \pm 7.18$  \\
        SMM           & $25.24 \pm 14.38$ & $60.11 \pm 18.87$ \\
        ICM-APT       & $65.00 \pm 27.77$ & $65.29 \pm 11.80$ \\
        MAXENT        & $17.46 \pm 11.68$ & $19.90 \pm 12.10$ \\
        RND           & $23.56 \pm 12.34$ & $47.71 \pm 18.07$ \\
        \textbf{ROVER}& $\mathbf{81.25 \pm 6.34}$ & $\mathbf{83.20 \pm 4.55}$ \\
        \bottomrule
    \end{tabular}
\end{table}

\section{Conclusion}
\label{sec:conclusion}

We introduced \algo{}, a reward-free pretraining method for sparse-reward settings where downstream rewards are not available for offline relabeling. In this regime, the main transferable object is the behavior induced by the pretrained policy. We therefore formulate reward-free pretraining as \emph{occupancy coverage}: learning a policy whose discounted occupancy is broadly distributed over the reachable state space.

Our main contribution is a theoretically and empirically motivated objective for behavioral transfer. By minimizing the self-similarity of the discounted occupancy measure in an RKHS, the objective recovers uniform coverage in the finite tabular case and encourages space-filling behavior in learned representation spaces. Further, we describe \algo{}, a concrete algorithm that optimizes this objective, using a learned resolvent world model and a sink-state augmentation that stabilizes exploration by balancing expansion into unseen regions with retention of previously covered states.

Empirically, \algo{} induces broader and more stable coverage than standard reward-free baselines in controlled sparse-navigation tasks, and provides effective initializations for downstream model-free RL. The stability of the algorithm is ensured thanks to the introduction of a virtual ``sink state'', which allows for controlling the exploration-exploitation trade-off. Our current study is intentionally limited to discrete action spaces and simple environments. This choice allows us to isolate the coverage objective and analyze its behavior directly, rather than presenting \algo{} as a fully scalable exploration algorithm.

Future work should study scalable representations that make occupancy coverage meaningful in high-dimensional domains, approximate implementations of the resolvent objective, and extensions to continuous actions, partial observability, and larger benchmarks. We view this work as a first step toward establishing occupancy coverage as a principled objective for reward-free behavioral transfer.

\section*{Acknowledgments}
We acknowledge ISCRA for awarding this project access to the LEONARDO supercomputer, owned by the EuroHPC Joint Undertaking, hosted by CINECA (Italy).
We gratefully acknowledge IIT and its HPC Team for their support and assistance on the IIT High Performance Computing Infrastructure.
This project received funding from the European Union’s Horizon Europe research and innovation program under grant agreement 101120237 (ELIAS) and ERC Advanced Grant LEO under grant agreement 101267213.

\bibliographystyle{unsrtnat}
\bibliography{bibliography}
\newpage
\appendix

\section{Connections to Information Geometry and Potential Theory} \label{app:connections}

\subsection{Connection to Rényi Entropy and Diversity.} \label{app:Reniy_diversity}
Our use of a Reproducing Kernel Hilbert Space (RKHS) naturally equips the state space with a notion of similarity, quantified by the kernel function $k(x, y) = \langle \phi(x), \phi(y) \rangle_{\mathcal{H}}$. Following the framework of~\cite{leinster2021maximum}, we use this similarity to define the \textit{typicality} of a state $x$ with respect to the occupancy measure $d(\pi)$ as the expected similarity to a random sample: 
\begin{equation*}
    \kappa(x) = \int k(x, y) \, d(\pi)(dy).
\end{equation*}
This measure of local concentration allows us to define the global ``spread" of the distribution. The \textit{diversity of order $q$} is defined as the generalized mean (of order $1-q$) of the inverse typicality (or ``atypicality"):
\begin{equation}
    D_q(d(\pi)) = \left( \int_{\spX} \left( \frac{1}{\kappa(x)} \right)^{1-q} d(\pi)(dx) \right)^{\frac{1}{1-q}}.
\end{equation}
In the specific case of $q=2$, this expression simplifies to the harmonic mean of the atypicality, which is the reciprocal of the expected similarity:
\begin{equation}\label{eq:diversity2}
    D_2(d(\pi)) = \frac{1}{\int_{\spX} \kappa(x) \, d(\pi)(x)} = \frac{1}{\|\mu(\pi)\|_{\mathcal{H}}^2}.
\end{equation}
Consequently, minimizing our objective $\|\mu(\pi)\|_{\mathcal{H}}^2$ is mathematically equivalent to maximizing the diversity of the state occupancy distribution, as measured by~\cref{eq:diversity2}. Furthermore, the logarithm of this quantity, $H_2(d(\pi)) = \log D_2(d(\pi))$, corresponds exactly to the \textit{Rényi entropy of order 2} (collision entropy) and as shown by~\cite{leinster2021maximum}, maximizing this entropy provides a canonical definition of the ``uniform distribution" in general metric spaces where the standard Lebesgue measure may not apply, thereby justifying our objective as a generalized uniformity-matching strategy.

\subsection{Connection to Potential Theory.}\label{app:potential_theory}
Our objective $\mathcal{J}(\pi)$ is also formally equivalent to the energy integral in potential theory~\cite{landkof1972foundations}. Interpreting the kernel $k(x,y)$ as a repulsive potential between $x$ and $y$, minimizing $\|\mu(\pi)\|_{\mathcal{H}}^2$ corresponds to finding the minimum energy configuration of the probability mass $d(\pi)$.

The coverage properties of the solution of such a problem have been studied for a number of different kernel (that is, potential) functions.
For singular potentials (e.g., Riesz kernels $k(\| x -  y\|)=\|x - y \|^{-s}$), the solution (equilibrium measure) is theoretically guaranteed to be the uniform distribution on the domain~\cite{landkof1972foundations, hardin2005minimal}.

\begin{example}[Diffusive vs. Ballistic Transport]
\label{ex:transport}
To illustrate how energy minimization drives exploration, we analyze a canonical 1D exploration problem. Consider an agent on the semi-infinite integer lattice $\mathbb{Z}_{\ge 0}$ starting at $x_0=0$. We compare the Riesz energy $\mathcal{J}_s(\pi)$ (using kernel $k(x,y) = |x-y|^{-s}$ for $s > 0$) of two distinct policies in the limit of long horizons ($\gamma \to 1$):
\begin{enumerate}
    \vspace{-.3truecm}
    \item \textbf{Diffusive Policy ($\pi_{diff}$):} A uniform random walk ($a \sim \text{Bernoulli}(0.5)$), reflecting at the origin.
        \vspace{-.1truecm}
    \item \textbf{Ballistic Policy ($\pi_{bal}$):} A directed policy ($a = +1$) moving constantly to the right.
        \vspace{-.3truecm}
\end{enumerate}
Intuitively, $\pi_{bal}$ achieves superior coverage because it strictly avoids backtracking, whereas $\pi_{diff}$ wastes time going backward half of the time. In \cref{app:energy-proof}, we derive the characteristic length $L$ of the state occupancy for both policies, showing that coverage scales as $L_{bal} \propto \epsilon^{-1}$ versus $L_{diff} \propto \epsilon^{-1/2}$, where $\epsilon = 1-\gamma$.

\begin{restatable}[Energy Scaling]{proposition}{transportProp}
\label{prop:transport_scaling}
As $\gamma \to 1$, the energies of the diffusive and ballistic policies scale as:
\begin{equation}
    \mathcal{J}_s(\pi_{diff}) \propto \epsilon^{s/2} \quad \text{and} \quad \mathcal{J}_s(\pi_{bal}) \propto \epsilon^{s}.
\end{equation}
Since $\epsilon$ is small and $s > 0$, we have $\epsilon^s \ll \epsilon^{s/2}$, implying $\mathcal{J}_s(\pi_{bal}) \ll \mathcal{J}_s(\pi_{diff})$.
\end{restatable}

This result demonstrates that our objective function~\cref{eq:objective} can be interpreted as the interaction energy of a cloud of particles, and it correctly identifies and penalizes the backtracking moves of random walks. By minimizing this objective, the agent is naturally driven toward the ballistic regime, maximizing the spatial spread of its trajectories.
\end{example}

\section{Proof of Proposition \ref{prop:transport_scaling}} \label{app:energy-proof}

\transportProp*

\begin{proof}
We derive the asymptotic scaling of the state occupancy measure $d_\pi$ and the resulting Riesz energy $\mathcal{J}_s(\pi)$ for both policies as $\gamma \to 1$. Let $\epsilon = 1 - \gamma$ denote the inverse effective horizon.

\textbf{Case A: Ballistic Policy ($\pi_{bal}$).}
The agent moves deterministically: $x_{t} = t$. The discounted state occupancy is given by the geometric distribution:
\begin{equation}
    d_{bal}(x) = (1-\gamma) \sum_{t=0}^\infty \gamma^t \delta(x - t) = \epsilon (1-\epsilon)^x.
\end{equation}
For small $\epsilon$, this approximates\footnote{This result can be obtained taking the logarithm of $(1 - \epsilon)^x$ and the first order Taylor approximation in $\epsilon$.} an exponential distribution $d_{bal}(x) \approx \epsilon e^{-\epsilon x}$. The characteristic width (scale) is $L_{bal} \sim \epsilon^{-1}$.

\textbf{Case B: Diffusive Policy ($\pi_{diff}$).}
The agent follows a symmetric random walk with reflection at $x=0$. The occupancy $d_{diff}(x)$ satisfies the stationary Bellman flow equation for $x > 0$:
\begin{equation}
    d(x) = \gamma \left[ \frac{1}{2} d(x-1) + \frac{1}{2} d(x+1) \right].
\end{equation}
Rearranging terms yields $\gamma d(x+1) - 2d(x) + \gamma d(x-1) = 0$. We seek a solution $d(x) \propto \lambda^x$ with $\lambda < 1$. The characteristic equation $\gamma \lambda^2 - 2\lambda + \gamma = 0$ yields roots $\lambda_{\pm} = (1 \pm \sqrt{1 - \gamma^2})/\gamma$. Selecting the stable root $\lambda_-$ and using $\sqrt{1-\gamma^2} \approx \sqrt{2\epsilon}$:
\begin{equation}
    \lambda_- \approx 1 - \sqrt{2\epsilon} \approx e^{-\sqrt{2\epsilon}}.
\end{equation}
Thus, $d_{diff}(x) \approx \sqrt{2\epsilon} e^{-x\sqrt{2\epsilon}}$, with characteristic width $L_{diff} \sim \epsilon^{-1/2}$. Now notice how both occupancies take the form of an exponential distribution $p(x) = \lambda e^{-\lambda x}$ supported on $[0, \infty)$, where $\lambda_{bal} = \epsilon$ and $\lambda_{diff} = \sqrt{2\epsilon}$.
We compute the Riesz energy integral explicitly using the substitution $u = \lambda x, v = \lambda y$:

\begin{align}
    \mathcal{J}_s(\pi) &= \int_0^\infty \int_0^\infty p(x) p(y) |x-y|^{-s} \, dx \, dy \\
    &= \int_0^\infty \int_0^\infty \left(\lambda e^{-\lambda x}\right) \left(\lambda e^{-\lambda y}\right) |x-y|^{-s} \, dx \, dy \\
    &= \lambda^2 \int_0^\infty \int_0^\infty e^{-u} e^{-v} \left| \frac{u}{\lambda} - \frac{v}{\lambda} \right|^{-s} \frac{du}{\lambda} \frac{dv}{\lambda} \\
    &= \lambda^s \underbrace{\int_0^\infty \int_0^\infty e^{-(u+v)} |u-v|^{-s} \, du \, dv}_{C_s}.
\end{align}
The integral $C_s$ converges to a constant depending only on $s$ (for $s < 1$). Thus, the energy scales directly with the decay rate: $\mathcal{J}_s(\pi) \propto \lambda^s$. Substituting the decay rates $\lambda$ derived in Part 1:
\begin{itemize}
    \item \textbf{Ballistic:} $\lambda_{bal} \sim \epsilon \implies \mathcal{J}_s(\pi_{bal}) \propto \epsilon^s$.
    \item \textbf{Diffusive:} $\lambda_{diff} \sim \epsilon^{1/2} \implies \mathcal{J}_s(\pi_{diff}) \propto (\epsilon^{1/2})^s = \epsilon^{s/2}$.
\end{itemize}
Comparing the two regimes as $\epsilon \to 0$:
\begin{equation}
    \frac{\mathcal{J}_s(\pi_{bal})}{\mathcal{J}_s(\pi_{diff})} \propto \frac{\epsilon^s}{\epsilon^{s/2}} = \epsilon^{s/2} \to 0.
\end{equation}
\end{proof}

\section{Proof of Theorem \ref{thm:dual_gradient}}
\label{app:dual_gradient_proof}

We derive the dual form of the gradient for a general transition operator with the sandwich structure $\trop = \Phi^\top G \Psi$. Starting from the primal gradient~\cref{eq:primal_gradient}:
\begin{equation}
\nabla_{\polop} \mathcal{J}(\pi) = 2\gamma(1-\gamma)^2 \trop^* R_\pi^* \mu(\pi) \otimes R_\pi \mu_0,
\end{equation}
where $R_\pi = (I - \gamma \trop \polop_{\pol})^{-1}$ is the resolvent operator.

\textbf{Step 1: Resolvent representation.} The key observation is that the resolvent admits the dual representation
\begin{equation}
(I - \gamma \trop \polop_{\pol})^{-1} = \gamma \Phi^\top (\Id - \gamma G M_\pi)^{-1} G \Psi \polop_{\pol} + I_{\mathcal{H}},
\end{equation}
where terms orthogonal to the data span contribute identity. To see this, note that the Neumann series $(I - \gamma \trop \polop)^{-1} = \sum_{t=0}^\infty (\gamma \trop \polop)^t$ can be ``folded'' using the sandwich structure: each power $(\trop \polop)^t = \Phi^\top (G M_\pi)^{t-1} G \Psi \polop$ for $t \geq 1$, where $M_\pi = \Psi \polop_{\pol} \Phi^\top$ captures the policy-induced coupling between data points.

Applying this to $\mu_0 = \Phi^\top \alpha_0$ yields:
\begin{equation}
R_\pi \mu_0 = \Phi^\top (\Id - \gamma G M_\pi)^{-1} \alpha_0.
\end{equation}

\textbf{Step 2: Adjoint action.} The adjoint $\trop^* = \Psi^\top G \Phi$ maps the occupancy embedding $\mu(\pi) = (1-\gamma)\Phi^\top (\Id - \gamma G M_\pi)^{-1} \alpha_0$ to:
\begin{equation}
\trop^* R_\pi^* \mu(\pi) = (1-\gamma) \Psi^\top G (\Id - \gamma M_\pi^\top G)^{-1} K_\phi (\Id - \gamma G M_\pi)^{-1} \alpha_0.
\end{equation}

\textbf{Step 3: Dual gradient form.} Evaluating the gradient at $(x,a)$ and using $\langle \psi(x,a), \Psi^\top v \rangle = \sum_i v_i \langle \psi(x,a), \psi(x_i, a_i) \rangle$, we obtain the stated form with coefficient vector $c_\pi = 2\gamma(1-\gamma)^2 G (\Id - \gamma M_\pi^\top G)^{-1} K_\phi (\Id - \gamma G M_\pi)^{-1} \alpha_0$. \qed

\section{Properties of the Augmented Transition Operator}
\label{app:sink_prop}

\subsection{Preliminaries - Discounted Occupancy Measure}\label{app:preliminaries}
We define the discounted state occupancy measure $\nu(\pi)$ as follows:
\begin{equation}
   d^\pi(x) = (1-\gamma)\sum_{t=0}^{\infty}\gamma^t
\underset{\substack{s_0 \sim \mu_0(\cdot), \\ a_t \sim \pi(\cdot | s_t), \\ s_{t+1} \sim \mathcal{T}(\cdot | s_t, a_t)}}{\mathbb{E}} [ \mathbb{1}(s_t = sx]
\end{equation}
where $(1-\gamma)$ serves as the normalization constant, and the expectation is taken over trajectories induced by the policy $\pi$ and the transition dynamics $\mathcal{T}$ (as defined in \cref{eq:objective}).

By leveraging the operator formulation introduced in \cref{eq:objective}, we can express the probability of visiting state $s$ at time step $t$ in a compact form. Expanding the expectation over the trajectory distribution yields:
\begin{align}
    \underset{\substack{s_0 \sim \mu_0(\cdot), \\ a_t \sim \pi(\cdot | s_t), \\ s_{t+1} \sim \mathcal{T}(\cdot | s_t, a_t)}}{\mathbb{E}} [ \mathbb{1}(s_t = s)]  &= \sum_{s_0 \in \mathcal{S}} \mu_0(s_0) \sum_{a_0 \in \mathcal{A}} \pi(a_0|s_0)\sum_{s_1 \in \mathcal{S}} \mathcal{T}(s_1|s_0, a_0) \ldots\sum_{a_{t-1} \in \mathcal{A}} \pi(a_{t-1}|s_{t-1}) \mathcal{T}(s|s_{t-1}, a_{t-1}) \\
    &= [( \trop\polop_{\pol})^t\mu_0)](s)
\end{align}
Here, the recursive application of the policy operator $\polop_{\pol}$ and transition operator $\trop$ captures the sequential evolution of the state density. Consequently, the occupancy measure can be written as the resolvent of the transition operator applied to the initial distribution.

Thus, we can plug this result into the discounted occupancy measure above:
\begin{equation}
   d^\pi(x) = (1-\gamma)\sum_{t=0}^{\infty} [(\gamma\trop\polop_{\pol})^t\mu_0](x) = (1-\gamma)[(\Id - \gamma \trop\polop_{\pol})^{-1}\mu_0] (x)
\end{equation}
where the last equality follows from $\trop$ and $\polop$ being Markov operators \cite{aliprantis2006infinite} ( $\|\trop\| = \|\polop\| = 1$), making the Neumann series convergent.

Consequently, we assume that for a given  $\phi: \mathcal{X} \to \mathbb{R}^z$ and a dataset $\Phi \in \mathbb{R}^{n\times z}$, we can rewrite $\mu_0 = \Phi^* \alpha_0$, where $\alpha_0$ is coefficient array. Recalling \cref{thm:dual_gradient} we assume to work with $\trop = \Phi^* G \Psi$ for some matrix $G$. Then, we can rewrite the discounted state space function as
\begin{align}
   d^\pi(x) &= (1-\gamma)\sum_{t=0}^{\infty} [(\gamma\trop\polop_{\pol})^t\mu_0](x) \\
    &= (1-\gamma)\sum_{t=0}^{\infty} [(\gamma\Phi^* \hat G\Psi\polop_{\pol})^t \Phi^*\alpha_0](x) \\
    &= (1-\gamma)\sum_{t=0}^{\infty} [\Phi^*(\gamma \hat G\Psi\polop_{\pol} \Phi^*)^t\alpha_0](x) \label{eq:unrolling_and_rerolling}
\end{align}
where in \cref{eq:unrolling_and_rerolling}, we unravel the sum and the exponential to bring in the right $\Phi^*$ and bring out the left.
We then define $M = \Psi\polop_{\pol} \Phi^* $ and, consequently, $d^\pi(x) = (1-\gamma)\sum_{t=0}^{\infty} [\Phi^*(\gamma B\Psi\polop_{\pol} \Phi^*)^t\alpha_0](x) $.
\subsection{Feature augmentation}
Here we show that under the augmented kernel definition, the estimated transition operator $\hat{\trop}_n$ maps the sink state to itself. Let $e_{\sink}$ be the embedding of the sink state $s_{\sink}$.  Thus, we set $\hat{\Phi} \in \mathbb{R}^{n \times z_x + 1}, \hat{\Psi} \in \mathbb{R}^{n \times z_{xa} + 1}$ as the datasets containing the state and the state-action pair embeddings augmented with an additional dimension, which is set to $0$.
The sink state embedding is defined as
$ e_\sink = \left(
\begin{smallmatrix}
0 \\
\vdots \\
0 \\
1
\end{smallmatrix}
\right) \in \mathbb{R}^{z_x+1}.$
In the following analysis, we define $\mathbb{1}_k$ as a column vector containing $k$-times the $1$ value.
The augmented version of embedding datasets $\bar{\Phi}$ end $\bar{\Psi}$ are defined in the following manner:
\begin{equation}
\label{eq:phi_augmented}
\bar{\Phi}^* =
\begin{bmatrix}
\hat{\Phi}^* - e \mathbb{1}^*_{z_{xa}+1} \hat{\Psi}^* & e
\end{bmatrix}
\in \mathbb{R}^{z_x+1 \times n+1}
\end{equation}
\begin{equation}\label{eq:psi_augmented}
\bar{\Psi} =
\begin{bmatrix}
\hat{\Psi}  \\
\mathbb{1}^*_{z_{xa}+1}
\end{bmatrix}
\in \mathbb{R}^{n+1\times z_{xa}+1}
\end{equation}
The coefficient matrix $\hat{G} = (K + \lambda \Id)^{-1}$ from~\cref{prop:augmented_sandwich_main} extends to the augmented setting as:
\begin{equation}
\bar{G} =
\begin{bmatrix}
\hat{G} & 0 \\
0 & 1
\end{bmatrix}
\in \mathbb{R}^{n+1 \times n+1}.
\end{equation}
We assume that the initial distribution can be rewritten as $\bar{\nu}_0 = \bar{\Phi}^* \bar{\alpha}_0$, where $ \bar{\alpha}_0 = \left(
\begin{smallmatrix}
1 \\
0 \\
\vdots \\
0 \\
\end{smallmatrix}
\right) \in \mathbb{R}^{z_x+1}. $

We also assume the normalization condition: there exists $\mathbf{1}_{\mathcal H}\in\mathcal H$ such that
\begin{equation}\label{eq:normalization_assumption}
    \langle \mathbf{1}_{\mathcal H}, \psi(x,a) \rangle_{\mathcal H} = 1,
\qquad \forall (x,a)\in\mathcal X\times\mathcal A.
\end{equation}

\begin{proposition}[Structure of the Augmented Operator $\bar \trop_n$]\label{prop:augmented_sandwich}
Under the normalization condition~\cref{eq:normalization_assumption}, the augmented transition operator $\bar{\trop}$ defined in~\cref{eq:augmented_cme_opt} admits the sandwich form
\begin{equation}
\bar{\trop} = \bar{\Phi}^\top \bar{G} \bar{\Psi},
\end{equation}
where $\bar{\Phi}$ and $\bar{\Psi}$ are the augmented feature matrices defined in~\cref{eq:phi_augmented,eq:psi_augmented}, and $\bar{G} = \mathrm{diag}(\hat{G}, 1)$ with $\hat{G} = (K + \lambda \Id)^{-1}$. Moreover, the policy-induced matrix inherits a block-diagonal structure:
\begin{equation}
\bar{M}_\pi = \mathrm{diag}(M_\pi, 1),
\end{equation}
where $M_\pi$ is the policy-induced matrix from the original (non-augmented) problem. Consequently, the dual gradient formula of \cref{thm:dual_gradient} applies directly with the $(n+1)$-dimensional augmented matrices.
\end{proposition}

\begin{proof}
The proof proceeds in three parts: (i) verifying the sandwich form $\bar{\trop} = \bar{\Phi}^\top \bar{G} \bar{\Psi}$, (ii) showing the sink state is absorbing, and (iii) establishing the block-diagonal structure of $\bar{M}_\pi$.

\paragraph{Part 1: Sandwich form}
By construction, $\bar{\trop} = \bar{\Phi}^\top \bar{G} \bar{\Psi}$ with $\bar{G} = \mathrm{diag}(\hat{G}, 1)$ follows directly from the definitions~\cref{eq:phi_augmented,eq:psi_augmented}.

\paragraph{Part 2: The sink state is absorbing}
We verify that the transition operator maps the sink state to itself. Before passing through $\bar{\trop}$, the sink state passes through $\polop_{\pol}$ yielding $e_{xa} \in \mathbb{R}^{z_{xa} +1}$.
\begin{align}
\bar{\trop} e_{xa} &= \bar{\Phi}^* \bar{G} \bar{\Psi} e_{xa} \\
 &= \bar{\Phi}^* \bar{G} \begin{bmatrix}
\hat{\Psi}\\
\mathbb{1}^*_{z_{xa}+1}
\end{bmatrix}
e_{xa} \\
 &= \bar{\Phi}^* \bar{G} \begin{bmatrix}
\cancel{\hat{\Psi} e_{xa}}  \\
\mathbb{1}^*_{z_{xa}+1}  e
\end{bmatrix} \label{eq:t-cancel-out} \\
 &= \bar{\Phi}^* \begin{bmatrix}
\hat{G} & 0 \\
0 & 1
\end{bmatrix} \begin{bmatrix}
0 \\
1
\end{bmatrix} \\
 &= \bar{\Phi}^* \begin{bmatrix}
0 \\
1
\end{bmatrix} \\
 &= \begin{bmatrix}
\hat{\Phi}^* - e \mathbb{1}^*_{z_{xa}+1}  \hat{\Psi}^* & e
\end{bmatrix} \begin{bmatrix}
0 \\
1
\end{bmatrix} \\
 &= e
\end{align}
This shows that the sink state is absorbing: $\bar{\trop} e_{xa} = e$.
In \cref{eq:t-cancel-out}, we used that $\hat{\Psi} e_{xa} = 0$ by orthogonality and $\mathbb{1}^*_{z_{xa}+1} e = 1$.

For completeness, we also verify the behavior on regular states:
\begin{align}
\bar{\trop} \hat{\psi}(x) &= \bar{\Phi}^* \bar{G} \bar{\Psi} \hat{\psi}(x) \\
&= \bar{\Phi}^* \bar{G} \begin{bmatrix}
\hat{\Psi}  \\
\mathbb{1}^*_{z_{xa}+1}
\end{bmatrix} \hat{\psi}(x) \\
&= \bar{\Phi}^*
\begin{bmatrix}
\hat{G} & 0 \\
0 & 1
\end{bmatrix}
\begin{bmatrix}
\hat{\Psi}\hat{\psi}(x,a) \\
\mathbb{1}^*_{z_{xa}+1}  \hat{\psi}(x,a)
\end{bmatrix} \\
&=
\begin{bmatrix}
\hat{\Phi}^* - e \mathbb{1}^*_{z_{xa}+1}  \hat{\Psi}^* & e
\end{bmatrix}
\begin{bmatrix}
\hat{G} \hat{\Psi}\hat{\psi}(x,a) \\
\mathbb{1}^*_{z_{xa}+1}  \hat{\psi}(x,a)
\end{bmatrix} \\
&=
\begin{bmatrix}
\hat\Phi^*\hat G \hat\Psi\hat\psi(x,a) - e \mathbb{1}^*_{z_{xa}+1}  \hat\Psi^*\hat G \hat\Psi\hat\psi(x,a) + e\mathbb{1}^*_{z_{xa}+1} \hat\psi(x,a)
\end{bmatrix} = (1), \, (2)\label{eq:constraints-on-1-and-cancellation} \\
(1) &\approx
\begin{bmatrix}
\hat\Phi^*\hat G \hat\Psi\hat\psi(x,a) + e\mathbb{1}^*_{z_{xa}+1} \hat\psi(x,a)
\end{bmatrix}, \, \text{if} \,\,\, \hat\Psi\hat\psi(x,a) \approx 0 \\
(2) &\approx
\begin{bmatrix}
\hat\Phi^*\hat G \hat\Psi\hat\psi(x,a)
\end{bmatrix}  \label{eq:second_derivation}
\end{align}
In \cref{eq:constraints-on-1-and-cancellation}, $\hat G = (\hat\Psi \hat\Psi^T + \lambda \Id)^{-1} \approx (\hat\Psi \hat\Psi^T)^{-1}$ for a small $\lambda$. Then, $e \mathbb{1}^*_{z_{xa}+1}  \hat\Psi^*\hat G \hat\Psi\hat\psi(x,a) \approx e \mathbb{1}^*_{z_{xa}+1}\hat\psi(x,a)$ and it leads to \cref{eq:second_derivation}.

\paragraph{Part 3: Block-diagonal structure of $\bar{M}_\pi$}
Recall from \cref{app:preliminaries} that $M = \Psi\polop \Phi^*$. We show that $\bar{M} = \bar{\Psi} \bar{\polop} \bar{\Phi}^*$ has block-diagonal structure, where $\bar{\polop} =
\begin{bmatrix}
\polop \\
0
\end{bmatrix}$. We obtain the following structure for $\bar{M}$:
\begin{align}
\bar{M} &= \bar{\Psi} \bar{\polop} \bar{\Phi}^* \\ 
&= \bar{\Psi}
\begin{bmatrix}
\polop  \\
0
\end{bmatrix}
\begin{bmatrix}
\Phi^* - e \mathbb{1}^*_{z_{xa}+1} \Psi^* & e
\end{bmatrix} \\ 
&= \bar{\Psi} \begin{bmatrix}
\polop\Phi^* - \polop e \mathbb{1}^*_{z_{xa}+1} \Psi^* & \polop e
\end{bmatrix} \\ 
&= \begin{bmatrix}
\Psi \\
\mathbb{1}^*_{z_{xa}+1}
\end{bmatrix} \begin{bmatrix}
\polop\Phi^* - \polop e \mathbb{1}^*_{z_{xa}+1} \Psi^* & \polop e
\end{bmatrix} \\ 
&= \begin{bmatrix}
\Psi \polop\Phi^* - \Psi \polop e \mathbb{1}^*_{z_{xa}+1} \Psi^* & \Psi \polop e  \\
\mathbb{1}^*_{z_{xa}+1} (\polop\Phi^* - \polop e \mathbb{1}^*_{z_{xa}+1} \Psi^*) & \mathbb{1}^*_{z_{xa}+1}  \polop e
\end{bmatrix} \\ 
\end{align}
We analyze the $M$ matrix point-wise, starting from the upper-left block.
\begin{align}
\bar{M}_{ij} &= \psi(x_i,a_i)\big(\polop \phi^*(x'_j) - \polop e \mathbb{1}^*_{z_{xa}+1} \psi^*(x_j,a_j)\big)
\end{align}
In the following, we proceed with the computations with the first part of the multiplication, which is $\psi(x_i,a_i)\polop \phi(x'_j) \rangle$
\begin{align}
\langle\psi(x_i,a_i)^*,\polop \phi(x'_j) \rangle &= \langle \phi (x_i) \otimes e_{a_i}, \sum_{a \in \mathcal{A}}\pi(a | x'_j) \phi(x'_j) \otimes e_a \rangle \\
&=  \cancel{\sum_{a \in \mathcal{A}}} \pi(a | x'_j) \langle \phi (x_i) , \phi(x'_j)  \rangle \otimes e_{a = a_i } \label{eq:cancel-sum}\\
&= \pi(a | x'_j) K(x_i, x'_j)
\end{align}

Here, we proceed with the second element of the above multiplication $\psi(x_i,a_i) \polop e \mathbb{1}^*_{z_{xa}+1} \psi(x_j,a_j)$:
\begin{align}
\psi(x_i,a_i) \polop e \mathbb{1}^*_{z_{xa}+1} \psi(x_j,a_j) &= \pi(a_i|e) K(x_i, e)\mathbb{1}^*_{z_{xa}+1} \psi (x_j,a_j) = 0 \label{eq:psi-p-e}\\
\bar{M}_{ij} = \pi(a | x'_j) K(x_i, x'_j) &= M_{ij}
\end{align}
In \cref{eq:cancel-sum}, the sum is zero for every $a \ne a_i$, due to the presence of $e_{a = a_i}$, which is basically a Dirac delta.
The upper-right block yelds to  $\pi(a_i|e) K(x_i, e) = 0$ as in \cref{eq:psi-p-e}, which is $0$ due to $K(x_i, e)$.
For the bottom-left block, the $(n+1, j)$ entry of $\bar{M}$ corresponds to the inner product between the $(n+1)$-th row of $\bar{\Psi}$ (which is $\mathbf{1}^\top$) and the $j$-th column of $\bar{\polop}\bar{\Phi}^*$. For a regular state $x'_j \in \mathcal{X}$:
\begin{align}
(\bar{M})_{n+1,j} &= \langle \mathbf{1}, \polop_{\pol}\bar{\phi}(x'_j)\rangle - \langle \mathbf{1}, \polop_{\pol}\bar{\phi}(e)\rangle \cdot \langle \mathbf{1}, \psi(x_j, a_j)\rangle \notag\\
&= \sum_{a \in \mathcal{A}} \pi(a|x'_j) \underbrace{\langle \mathbf{1}, \psi(x'_j, a)\rangle}_{=1} - \sum_{a \in \mathcal{A}} \pi(a|e) \underbrace{\langle \mathbf{1}, \bar{\psi}(e, a)\rangle}_{=1} \cdot \underbrace{\langle \mathbf{1}, \psi(x_j, a_j)\rangle}_{=1} \notag\\
&= 1 - 1 = 0, \notag
\end{align}
where we used the normalization assumption $\langle \mathbf{1}_{\mathcal{H}}, \psi(x,a)\rangle = 1$ for all $(x,a)$.


This establishes $\bar{M}_\pi = \mathrm{diag}(M_\pi, 1)$, completing the proof.
\end{proof}

\newpage

\section{Projected-Occupancy Coverage}
\label{app:projected_occupancy_coverage}

This appendix gives the general form of the position-focused coverage objective used in the PointMaze experiment \cref{sec:pointmaze}. Let $\mathcal X$ be a Markov state space, $\mathcal Z$ a space of coverage variables, and $g:\mathcal X\to\mathcal Z$ a measurable projection. For a policy $\pi$ and discount $\gamma\in(0,1)$, its normalized discounted state occupancy is
\begin{equation}
d^\pi(B)=(1-\gamma)\sum_{t=0}^{\infty}\gamma^t
\Pr_\pi(s_t\in B),\qquad B\subseteq\mathcal X.
\label{eq:appendix_full_occupancy}
\end{equation}
The occupancy over the selected variables is the pushforward $q^\pi=g_{\#}d^\pi$, meaning
\begin{equation}
q^\pi(C)=d^\pi(g^{-1}(C))
=d^\pi(\{s:g(s)\in C\}),\qquad C\subseteq\mathcal Z.
\label{eq:appendix_pushforward}
\end{equation}
This definition marginalizes out state components omitted by $g$. In PointMaze, $g(s)=(x,y)$, so velocity affects the policy and dynamics but not which position is counted.

Let $k_{\mathrm c}$ be a positive-definite kernel on $\mathcal Z$ with feature map $\varphi_{\mathrm c}$ into an RKHS $\mathcal H_{\mathrm c}$, such that $k_{\mathrm c}(z,\tilde z)=\langle\varphi_{\mathrm c}(z),\varphi_{\mathrm c}(\tilde z)\rangle_{\mathcal H_{\mathrm c}}$. The mean embedding of the projected occupancy is
\begin{align}
\mu_{\mathrm c}^{\pi}
&=\int_{\mathcal Z}\varphi_{\mathrm c}(z)\,q^\pi(dz) \\
&=\int_{\mathcal X}\varphi_{\mathrm c}(g(s))\,d^\pi(ds),
\label{eq:appendix_projected_embedding}
\end{align}
where the second line follows from the definition of the pushforward. Squaring its RKHS norm gives the coverage objective:
\begin{align}
\mathcal J_{\mathrm c}(\pi)
&:=\|\mu_{\mathrm c}^{\pi}\|_{\mathcal H_{\mathrm c}}^2 \\
&=\int_{\mathcal Z}\!\int_{\mathcal Z}
k_{\mathrm c}(z,\tilde z)\,q^\pi(dz)\,q^\pi(d\tilde z) \\
&=\int_{\mathcal X}\!\int_{\mathcal X}
k_{\mathrm c}(g(s),g(\tilde s))\,d^\pi(ds)\,d^\pi(d\tilde s).
\label{eq:appendix_projected_coverage_objective}
\end{align}
The last line follows by applying the pushforward identity to both integrals. Minimizing this self-similarity objective favors occupancy mass spread across points that are dissimilar under the selected kernel.

\paragraph{Finite-sample occupancy and coefficient vector}
Given transition data $\{(s_i,a_i,s_i')\}_{i=1}^n$, let $\phi_{\mathrm d}(s)$ be the dynamics feature, $e_a$ the one-hot action feature, and $\psi_{\mathrm d}(s,a)=\phi_{\mathrm d}(s)\otimes e_a$. Define the feature matrices
\begin{equation}
\Phi_{\mathrm d}
=\begin{bmatrix}\phi_{\mathrm d}(s'_1)^\top\\ \vdots\\ \phi_{\mathrm d}(s'_n)^\top\end{bmatrix},
\qquad
\Psi_{\mathrm d}
=\begin{bmatrix}\psi_{\mathrm d}(s_1,a_1)^\top\\ \vdots\\ \psi_{\mathrm d}(s_n,a_n)^\top\end{bmatrix}.
\label{eq:appendix_data_matrices}
\end{equation}
With ridge parameter $\lambda>0$, define
\begin{equation}
G=(\Psi_{\mathrm d}\Psi_{\mathrm d}^{\top}+\lambda I)^{-1},
\qquad
\widehat T=\Phi_{\mathrm d}^{\top}G\Psi_{\mathrm d}.
\label{eq:appendix_transition_estimate}
\end{equation}
Let $\polop_\pi$ map state features to state-action features and set
$M_\pi=\Psi_{\mathrm d}\polop_\pi\Phi_{\mathrm d}^{\top}$.
Represent the initial-state embedding in the successor-feature dictionary as
$\Phi_{\mathrm d}^{\top}\alpha_0$. The resolvent coefficient vector is
\begin{equation}
\beta_\pi=(I-\gamma G M_\pi)^{-1}\alpha_0.
\label{eq:appendix_resolvent_coefficients}
\end{equation}
Let $\Phi_{\mathrm c}$ contain coverage features of the same successor samples: $\phi_{\mathrm c}(g(s_i'))$ for state observations or $f_{\mathrm c}(o_i')$ for pixels. Define the coverage Gram matrix $K_{\mathrm c}=\Phi_{\mathrm c}\Phi_{\mathrm c}^{\top}$. Then
\begin{align}
\widehat\mu_{\mathrm c}^{\pi}
&=(1-\gamma)\Phi_{\mathrm c}^{\top}\beta_\pi,\\
\widehat{\mathcal J}_{\mathrm c}(\pi)
&=(1-\gamma)^2\beta_\pi^{\top}K_{\mathrm c}\beta_\pi.
\label{eq:appendix_finite_coverage_objective}
\end{align}
Differentiating through the resolvent gives the dual coefficient vector used to form the policy-improvement signal:
\begin{equation}
\boxed{
c_{\pi}^{(\mathrm c)}
=2\gamma(1-\gamma)^2
G(I-\gamma M_\pi^{\top}G)^{-1}K_{\mathrm c}\beta_\pi .
}
\label{eq:appendix_coverage_coefficient}
\end{equation}
Equivalently, substituting $\beta_\pi$ gives
\begin{equation}
c_{\pi}^{(\mathrm c)}
=2\gamma(1-\gamma)^2
G(I-\gamma M_\pi^{\top}G)^{-1}
K_{\mathrm c}(I-\gamma G M_\pi)^{-1}\alpha_0 .
\label{eq:appendix_coverage_coefficient_expanded}
\end{equation}
The matrices $G$, $M_\pi$, and $K_{\mathrm c}$ are $n\times n$; $\alpha_0$, $\beta_\pi$, and $c_{\pi}^{(\mathrm c)}$ are vectors in $\mathbb R^n$. The dynamics features determine $G$ and $M_\pi$, while the selected coverage features determine $K_{\mathrm c}$. For clarity, in this analysis we omit the sink-state augmentation. It can be incorporated by applying the construction of \cref{app:sink_prop} to the dynamics operator and extending the coverage feature space with a sink component orthogonal to the features of physical states.

\newpage

\section{Algorithm}
\subsection{Simplified}

\begin{algorithm}[h!]
\caption{\algo{} (Simplified)}
\label{alg:rover_simple}
\begin{algorithmic}[1]
\REQUIRE Discount $\gamma$, step size $\eta$, regularization $\lambda$
\STATE Initialize policy $p=p_0$, replay buffer $\mathcal{D} = \emptyset$
\FOR{each episode}
    \STATE Collect transitions using $p$, add to $\mathcal{D}$
    \STATE \texttt{// 1. Representation Learning}
    \STATE Update encoder $\phi_\theta$ via contrastive loss~\cref{eq:l_nce} on $\mathcal{D}$
    \STATE \texttt{// 2. World Model Estimation}
    \STATE Compute $\bar \trop_n$ via (sink-biased) regression~\cref{eq:augmented_cme_opt} on $\mathcal{D}$
    \STATE \texttt{// 3. Policy Mirror Descent}
    \STATE Initialize $\pi_0 = 1/|\spA|$, $C = 0$
    \FOR{$t = 1, \ldots, T_{\text{PMD}}$}
        \STATE Compute gradient coefficients $c_{\pi_t}$~\cref{thm:dual_gradient}
        \STATE $C \gets C + c_{\pi_t}$
        \STATE $\pi_{t+1}(\cdot|x) \gets \text{softmax}\bigl(-\eta \langle \psi(x,\cdot), \Psi_n^\top C \rangle\bigr)$
    \ENDFOR
    \STATE Update $p \gets \pi_{T_{\text{PMD}}}$
\ENDFOR
\end{algorithmic}
\end{algorithm}

\subsection{Full Algorithm}
\begin{algorithm}[]
\caption{\algo{} (Full Details)}
\label{alg:dist_matching}
\begin{algorithmic}[1]

\REQUIRE
discount factor $\gamma\in(0,1)$, step size $\eta>0$, kernel function $k(x, x') = \langle\phi_\theta(x),\phi_\theta(x')\rangle$ with learned feature map $\phi_\theta:\spX \to \mathbb{R}^z$, initial weights $C_0 = 0\in\mathbb{R}^{m \times |\spA|}$, Linear Layer W. \\
$\texttt{T}_{\text{steps}}, \texttt{PMD}$: training steps, PMD steps

\FOR{$t = 1$ to $\texttt{T}_{\text{steps}}$}
    \STATE \textbf{compute} $H_t \in \mathbb{R}^{1 \times n}$ such that $H_{j} = k(x_t,x_j)$
    \STATE $\boldsymbol{a}_t \sim \operatorname{Categorical}\!\left(\operatorname{softmax}(-\eta H_t C_t)\right)$
    \STATE $\boldsymbol{x}_{t+1} \sim P(\cdot \mid \boldsymbol{x}_t, \boldsymbol{a}_t)$
    \STATE $\mathcal{D} \gets \mathcal{D} \cup (\boldsymbol{x}_t,\boldsymbol{a}_t,\boldsymbol{x}_{t+1})$
    \STATE $\phi_\theta \gets$ \textsc{UpdateEmbedders}$(\mathcal{D}, \phi_\theta)$
    \STATE $C_t \gets$ \textsc{UpdateActor}$(\mathcal{D} , \phi)$
\ENDFOR

\STATE
\FUNCTION{\textsc{UpdateEmbedders}$(\mathcal{D})$}{}
    \STATE Sample minibatch $\{(x_i, a_i, x'_i)\}_{i=1}^B \sim \mathcal{D}$
    \STATE Compute latent states $z_i = \phi_\theta(x_i)$
    \STATE Compute target latents $z'_i = (\phi_\theta(x'_i))|_{\text{stop\_grad}}$
    \STATE Predict next latent $\hat{z}'_i = W(z_i \otimes e_{a_i})$
    \STATE $\mathcal{L}_{\text{NCE}} = -\frac{1}{B} \sum_{i=1}^B \log \frac{\exp(\langle\bar{\hat{z}}'_i,\bar{z}'_i\rangle)}{\sum_{j=1}^B \exp(\langle\bar{\hat{z}}'_i,\bar{z}'_j\rangle)}$ \COMMENT{Contrastive (normalized) \cref{eq:l_nce}}
    \STATE Update embedder $\phi_\theta$ and $W$ to minimize $\mathcal{L}_{\text{NCE}}$
    \STATE \textbf{Return} $\phi_\theta$
\ENDFUNCTION

\STATE
\FUNCTION{\textsc{UpdateActor}$(\mathcal{D}, \phi)$}{}
    \STATE \texttt{/* Policy Mirror Descent */}
    \STATE Sample minibatch $\{(x_i, a_i, x'_i)\}_{i=1}^m \sim \mathcal{D}$
    \STATE \textbf{let} $k(x,z) = \phi(x)^\top \phi(z)$
    \STATE \textbf{let} $E\in\mathbb{R}^{m \times |\spA|}$ with rows $E_i = \text{\textsc{OneHot}}_{|\spA|}(a_i)$
    \STATE \textbf{let} $K \in \mathbb{R}^{m\times m}$ such that $K_{ij} = k(x'_i,x'_j)$
    \STATE \textbf{let} $H \in \mathbb{R}^{m\times m}$ such that $H_{ij} = k(x_i,x'_j)$
    \STATE \textbf{let} $\hat{\Phi} \in \mathbb{R}^{m \times z_x+1}$ such that $\hat{\Phi} = \{{\phi}(x^\prime_i)\}_{i=1}^m$
    \STATE \textbf{let} $\hat{\Psi} \in \mathbb{R}^{m \times z_{xa}+1}$ such that $\hat{\Psi} = \{{\phi}(x_i)\otimes e_a\}_{i=1}^m$
    \STATE \textbf{compute} $\tilde{\Phi}$
    \STATE \textbf{compute} $G = (\hat{\Psi} \hat{\Psi}^* + \lambda \Id)^{-1}$
    \STATE \textbf{compute} $\alpha_0 \in\mathbb{R}^m$ such that $\mu_0 = \hat{\Phi}^* \alpha_0$
    \STATE $\bar{G} \gets \begin{pmatrix} G & 0 \\ 0 & 1 \end{pmatrix}, \bar{\alpha}_0 \gets (\alpha_0^*, 1)^*$
    \STATE \texttt{/* Policy Mirror Descent Updates */}
    \STATE $C_0 = 0$
    \FOR{$t=0,1,\ldots,\texttt{PMD} - 1$}
        \STATE $\pi_{t+1} \gets$ \textsc{softmax}$\Big(-\eta H^* C_t\Big) \in \mathbb{R}^{m\times|\spA|}$ \COMMENT{Closed Form Policy Update}
        \STATE $M_{\pi_{t + 1}} = H \odot (E\pi_{t + 1}^{*}) \in\mathbb{R}^{m \times m}$

        \STATE $\bar{M}_{\pi_{t + 1}} \gets \begin{pmatrix} M_{\pi_{t+1}} & 0 \\ 0 & 1 \end{pmatrix}$ 
        
        \STATE $U_{\pi_{t+1}} \gets (I - \gamma \bar G\bar{M}_{\pi_{t + 1}})^{-*}\bar{\Phi}$
        \STATE $c \gets 2\gamma(1-\gamma)^2\,\bar{G}^{-*} U_{\pi_{t+1}} U_{\pi_{t+1}}^* \bar{\alpha}_0$ \COMMENT{Closed Form Gradient Computation for the MMD loss}
        \STATE $C_{t+1} \gets C_t + {\rm diag}(c) E$
    \ENDFOR
    \STATE \textbf{Return} $C_{\texttt{PMD}}$
\ENDFUNCTION
\end{algorithmic}
\end{algorithm}

\newpage
\section{More on the Experiments}
\subsection{Environment Details} \label{app:envs-info}
We evaluate \algo{} on sparse-reward navigation environments designed to isolate coverage and bottleneck-exploration behavior. 
We designed different goal navigation configurations to experiment in a controllable settings. 
All the following environments can give discrete observations (i.e., one-hot encoding) or pixel-based observations, which consist of an RGB image $84 \times 84$. The location information is \textbf{hidden} in both types of observation. The action space is discrete with four actions: up, down, left and right. The reward is $-1$ per step, to encourage the agent to find the shortest path, and the episode end when the agent reaches the target or it reaches $H$ steps.

\paragraph{Middle Room}
The main diagnostic domain is \textit{Middle Room}, where the agent starts from a central room connected through narrow single-cell corridors to multiple surrounding rooms. 
This layout makes it possible to distinguish transient discovery from stable occupancy coverage, since a policy must repeatedly traverse bottlenecks to induce broad aggregate behavior. Here, the states are 129, and the horizon is $H=15$.

\begin{figure}[h!]
    \centering
    \includegraphics[width=0.5\linewidth]{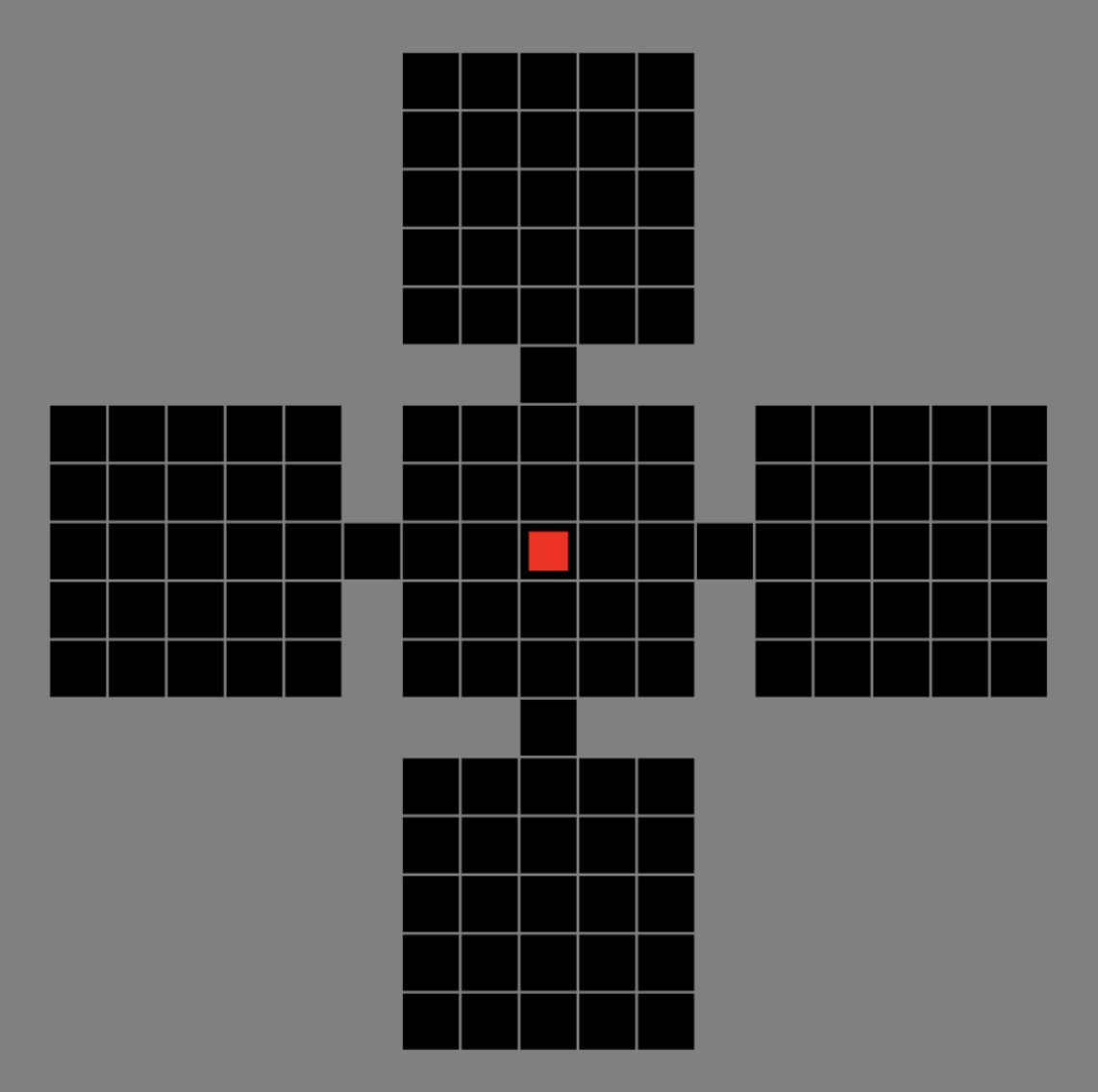}
    \caption{Snapshot of the \textit{Middle Room} environment}
    \label{fig:middle-room}
\end{figure}

\paragraph{Maze} We experimented also in a \textit{Maze} setting. In this case, the $\spX = |108|$ and the horizon $H=128$. 
\begin{figure}[h!]
    \centering
\includegraphics[width=1\linewidth]{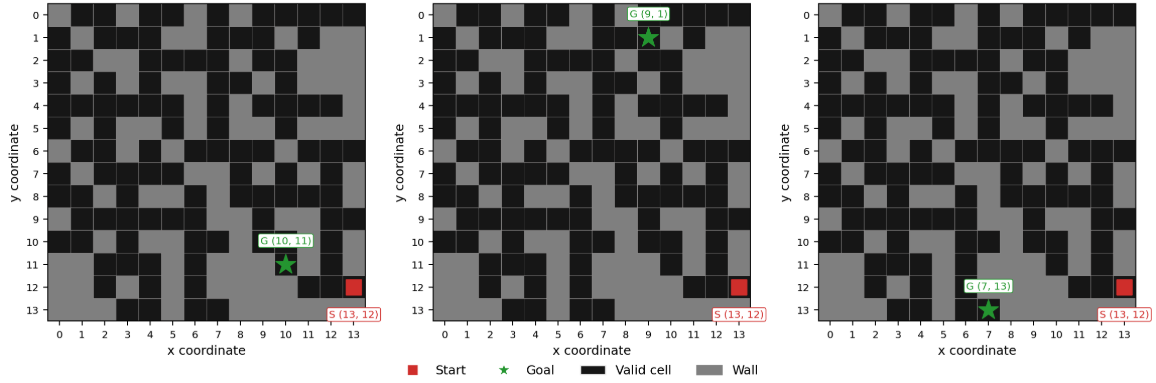}
    \caption{Snapshots of the Sparse Reward Maze environment.}
    \label{fig:maze_examples}
\end{figure}

\paragraph{Two Rooms and Multi Rooms} In the appendix, we extend our evaluation to other configurations: \textit{Two-Rooms} (\cref{fig:env_examples}, left) and \textit{Multi-Rooms }(\cref{fig:env_examples}, right) with goals placed at different distances from the start state. The state space and horizon are  respectively $|\spX|=99\,, H = 80$ and $|\spX|=109\,, H = 300$. \textit{Two-Rooms}  serve as a simple environment to see that already in the simplest bottlenecked configuration, we see a lack of performance of the baselines. The \textit{Multi-Rooms } is interesting in that all the baselines explore all the rooms effectively, but the learnt behaviour is not smart enough to probabilistically see all of them, while this is the case of \algo.  
\begin{figure}[h!]
    \centering
\includegraphics[width=1\linewidth]{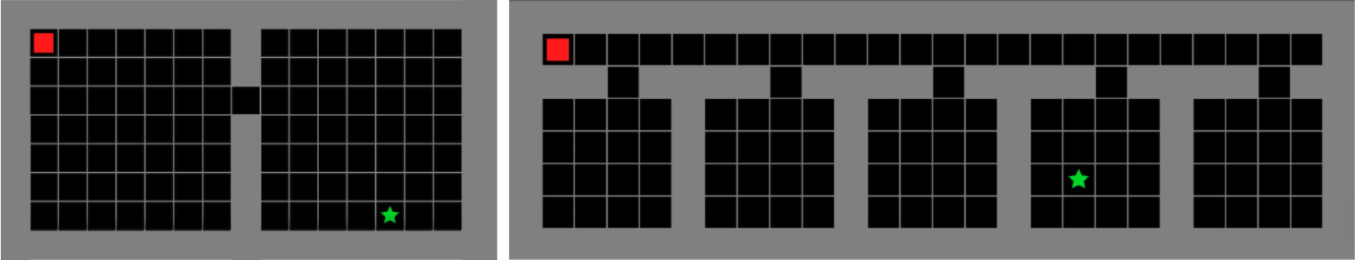}
    \caption{Snapshots of the Sparse Reward Navigation environments. The agent is depicted as the red square, while the goal is the green star. On the left, \textit{Two Rooms}. On the right, \textit{Multi-Room with a Single Corridor}}
    \label{fig:env_examples}
\end{figure}


\subsection{More on the Behaviour induced by Reward-free Objectives}
We analyze the behavior induced by different reward-free pretraining objectives in the \textit{Middle Room} environment as described already in \cref{sec:induced-behaviour}. 
 We compare \algo{} against Random Network Distillation (RND)~\citep{burda2018rnd}, State Marginal Matching (SMM)~\citep{lee2019smm}, Unsupervised Active Pre-Training (APT)~\citep{liu2021apt}, MAXENT~\citep{hazan2019maxent}, and Contrastive Intrinsic Control (CIC)~\citep{laskin2022CIC}.

We visualize each method by sampling $50$ trajectories from policy snapshots taken at initialization, at two intermediate checkpoints, and at the end of pretraining. 
For skill-based algorithms, each trajectory is computed by sampling a latent skill vector $z$ from a random distribution, following the codebase provided by \cite{laskin2022CIC, laskin2021urlb}.

\begin{figure}
    \centering
    \includegraphics[width=0.8\linewidth]{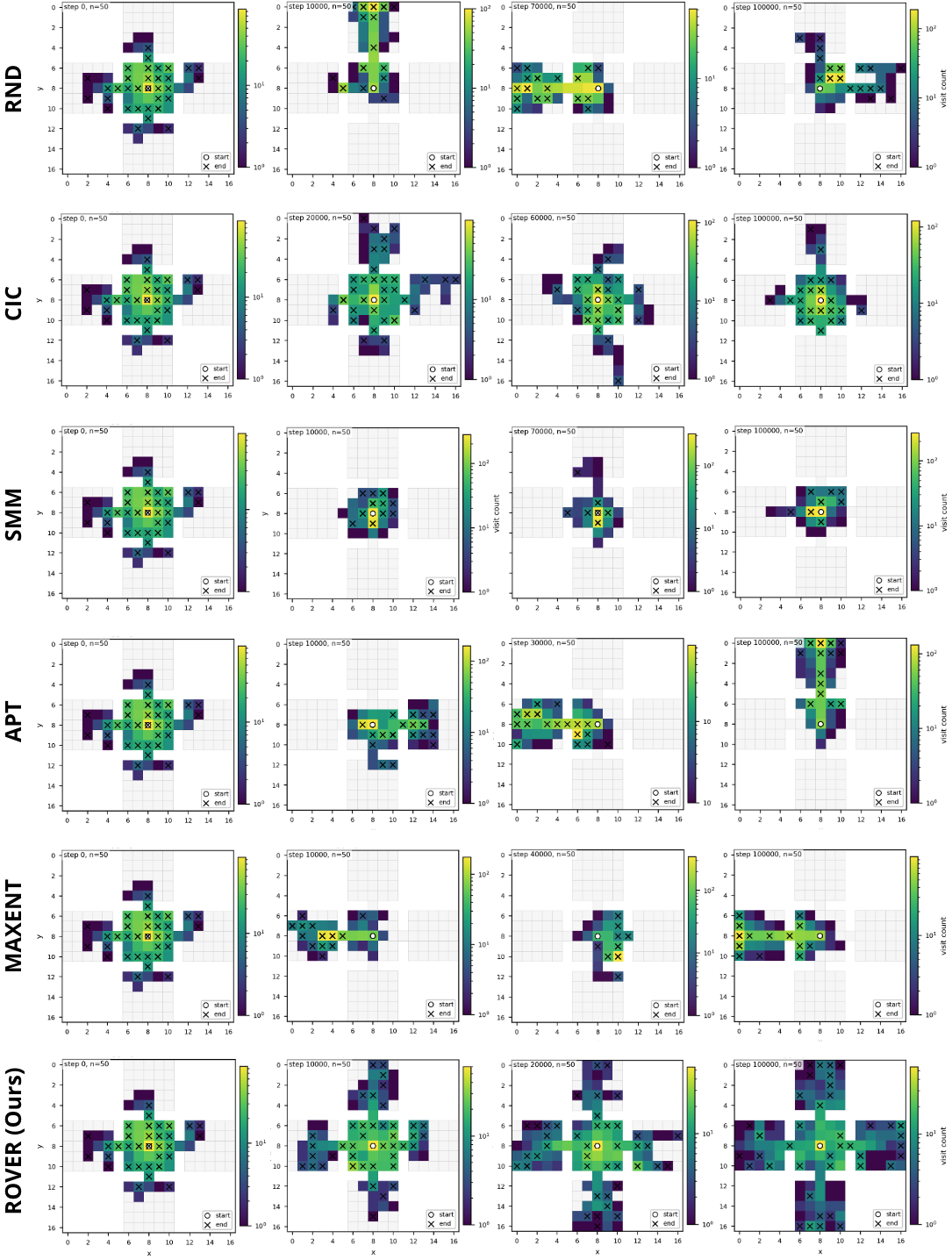}
    \caption{\textbf{Middle Room exploration during reward-free pretraining.} We sample 50 trajectories from policy snapshots at initialization, two intermediate checkpoints, and the end of pretraining for \algo{} and reward-free baselines. While several methods discover diverse states over training, their individual policies often collapse to localized occupancy; in contrast, effective transfer requires a final policy that broadly covers the state space.}
    \label{fig:induced_behaviour}
\end{figure}

\subsection{Finetuning the Pretrained Policies}
\label{sec:tabular-exps}
We evaluate the performance of finetuned policies in the tabular case with discrete state and action spaces in the \textit{Two-Rooms} environment. 

We compare our pre-training method against several unsupervised baselines: Random Network Distillation (RND) \cite{burda2018rnd}, State Marginal Matching (SMM) \cite{lee2019smm}, Unsupervised Active Pre-Training (APT) \cite{liu2021apt}, and Contrastive Intrinsic Control (CIC) \cite{laskin2022CIC}. 
For the supervised fine-tuning phase, we employ the Deep Deterministic Policy Gradient (DDPG) algorithm \cite{lillicrap2015ddpg} as done in \cite{laskin2021urlb}, adapted here to the discrete action space. 
We also include a ``from scratch'' baseline where DDPG is trained without any pre-training.

In \cref{fig:tabular_exps}, our method significantly outperforms all the baselines and the scratch initialization, achieving superior sample efficiency and final performance within the limited interaction budget. 
With this experiment, we highlight the failure of the baseline transfer behaviour even in the simplest setting.
As evidenced in \cref{fig:tabular_exps}, during the exploration phase, our policy generates a rich learning signal by frequently visiting the goal and populating the buffer with diverse transitions.

In the end, this pre-training overhead is amortized across all potential downstream tasks, and a further analysis is provided in App.~\ref{app:overhead}.
\begin{figure}[h!]
    \centering
    \includegraphics[width=0.8\linewidth]{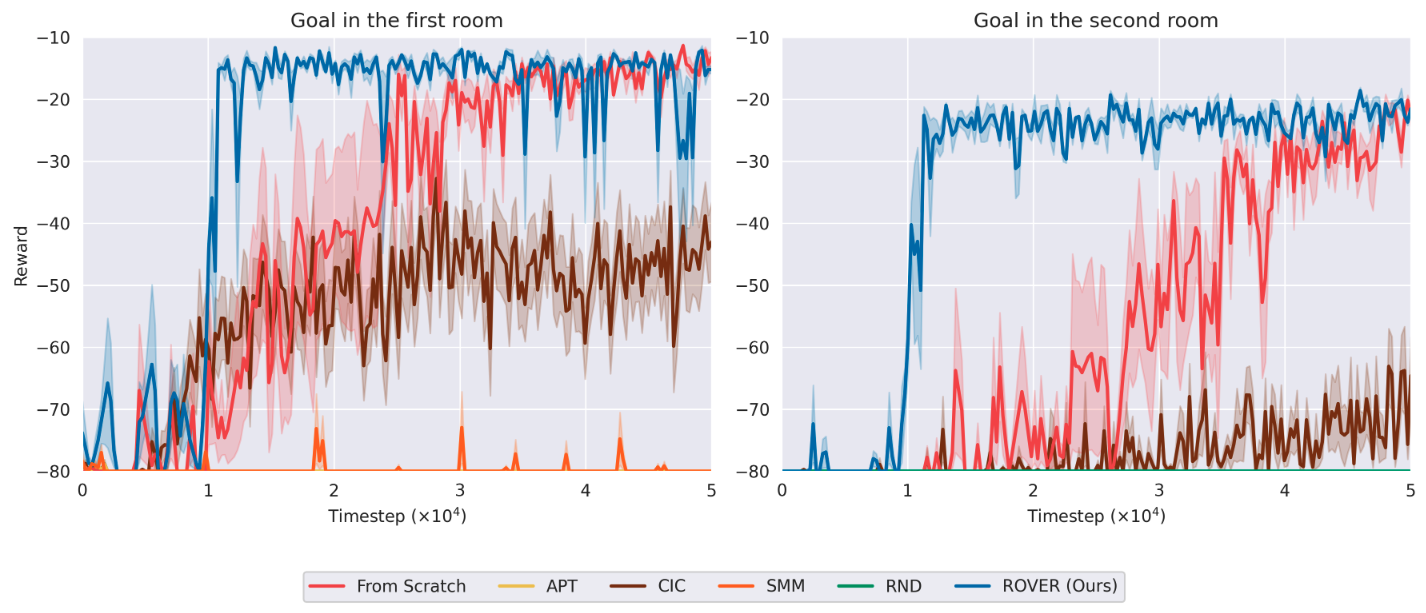}
    \caption{In the plots, the learning curve of DDPG using different policy initialization. The environment is \textit{Two-Rooms} of dimensions $7\times7$ connected by a corridor, which is a single state. In the left plot, the goal is set in the first room (i.e., simpler to discover), while on the right, the goal is in the second room with respect to the initial position. In the plot, we show the mean and standard error among five random seeds.}
    \label{fig:tabular_exps}
\end{figure}

\subsection{Initialization on Different Model-free Algorithms}\label{app:model-free_init}
We subsequently extend our evaluation to the \textit{Multi-Rooms} and \textit{Two-Rooms} environment in an image-based setting, where the agent receives $84 \times 84$ RGB images as input. In this context, we apply our initialization strategy on top of DDPG and Soft Actor-Critic (SAC) \cite{haarnoja2018sac}. In this high-dimensional setting, our initialization strategy consistently enhances sample efficiency compared to standard initializations, effectively accelerating the learning of the downstream policy. 
In \cref{fig:sac_ddpg}, we show the results for the \textit{Two Rooms} environment with the goal located in the first room (left) and second room (right). The bottom row displays results for the multi-room environment, with the goal positioned in the fourth room (left) and fifth room (right).
We observe that \algo\ initialization is compatible with both downstream algorithms, accelerating convergence in both cases. 
Notably, unlike in the tabular setting, DDPG fails to make progress in the \textit{Two Rooms} environment (specifically when the goal is in the second room) given the limited interaction budget. 
We attribute this failure to the difficulty of learning the encoder from images from scratch. 
In this context, \algo\ significantly improves sample efficiency by transferring a pre-trained encoder together with the initial policy.

\begin{figure}[h!]
    \centering

    \includegraphics[width=1\linewidth]{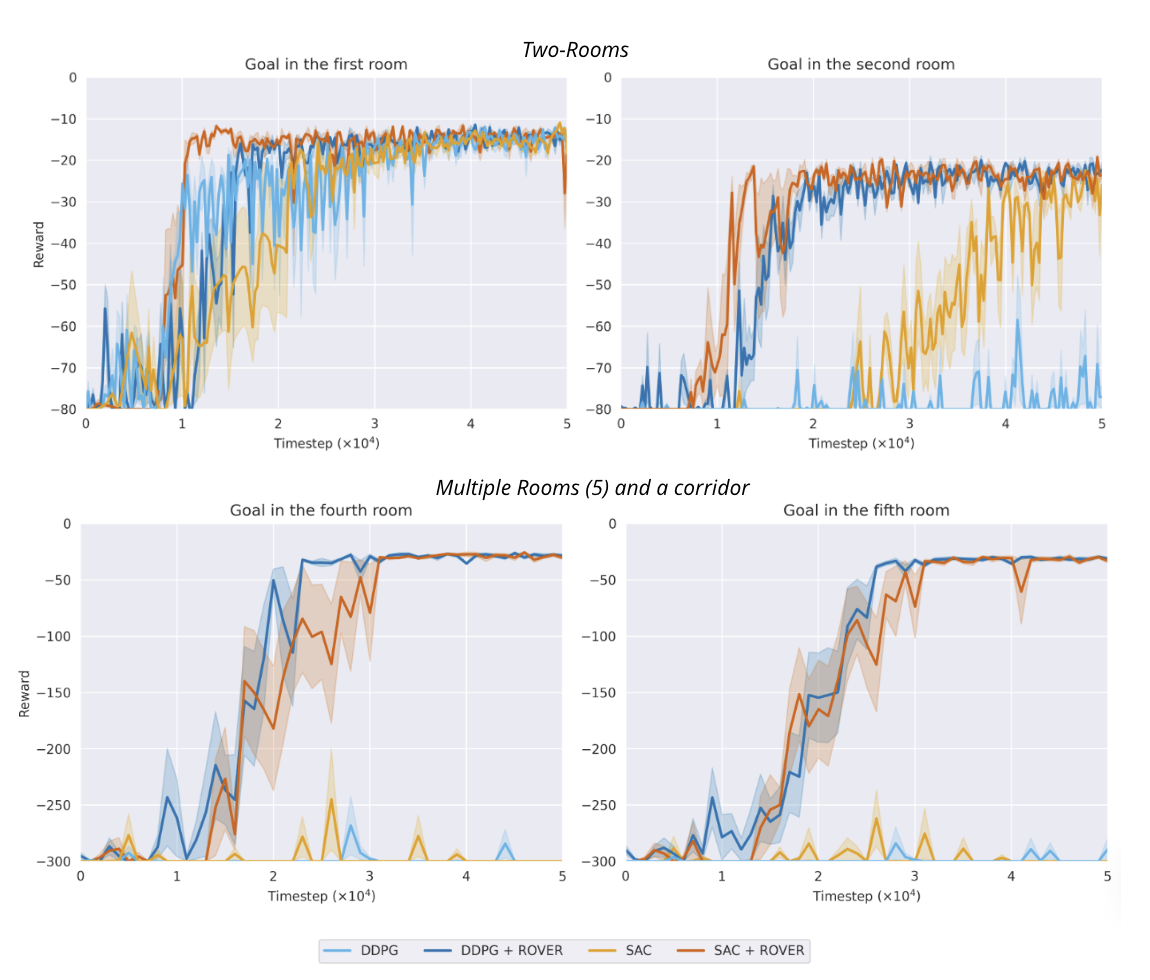}
    \caption{Comparison of learning curves for DDPG and SAC initialized with \algo. \textbf{Top row:} \textit{Two Rooms} environment with the goal located in the first room (left panel) and second room (right). \textbf{Bottom row:} five room 5 connected by a corridor environment with the goal in the fourth room (left) and fifth room (right).  We show the mean and standard error among five seeds.}
    \label{fig:sac_ddpg}
        \vspace{-.4truecm}
\end{figure}

\subsection{How ROVER Affects Sample Efficiency When Learning to Reach Different Goal Locations}

We conduct an ablation study to evaluate the sensitivity of convergence time to the distance between the starting position and the goal. To isolate the effects of policy initialization, we utilize a tabular setting with one-hot state encodings; this ensures that any performance gains are attributable solely to the transferred policy rather than feature learning.
As illustrated in \cref{fig:long-training}, the performance of standard DDPG degrades significantly as the goal distance increases, highlighting the difficulty of exploration in larger state spaces. 
In contrast, agents initialized with \algo\ demonstrate convergence behavior that is slightly variant to the target's location, effectively neutralizing the exploration penalty typically associated with distant goals.

\begin{figure}[h!]
    \centering
    \includegraphics[width=1\linewidth]{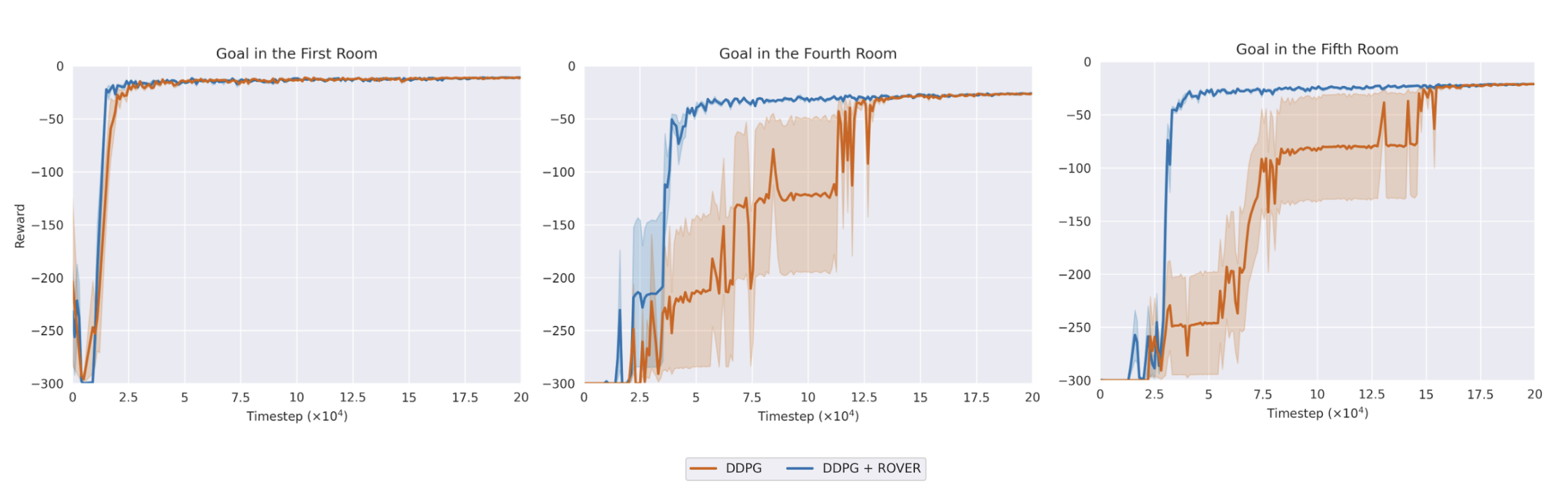}
    \caption{Comparison of learning curves for DDPG versus DDPG initialized with \algo\ in the multi-room environment. The plots show performance for goals located in the first room (left), fourth room (center), and fifth room (right). We show the mean and standard error among five seeds.}
    \label{fig:long-training}
        \vspace{-.4truecm}
\end{figure}

\subsection{Scaling Analysis of Pre-training Overhead} \label{app:overhead}

To assess the computational cost of our proposed pre-training phase, we analyzed the number of interaction steps required to achieve substantial state coverage across increasingly complex environments. We conducted this experiment in the \textit{Multi-Room} domain using one-hot state encodings, varying the number of rooms to increase the total state space size $|S|$.

We define the pre-training budget $T$ as the number of timesteps required for the policy to visit at least $95\%$ of the available states within a 20-episode evaluation window. Our objective is to determine the \textit{effective overhead} per potential downstream task. Since a single unsupervised pre-training phase facilitates learning for any goal location $g \in S$, the amortized cost per task is given by $T / |S|$.

The results, averaged over 5 seeds, are presented in \cref{tab:overhead_scaling}. We observe that the total interactions required ($T$) scale approximately linearly with the size of the state space. Consequently, the amortized overhead remains remarkably stable---roughly $65$--$75$ steps per potential goal state---even as the environment complexity triples.

\begin{table}[h]
    \centering
    \caption{Analysis of pre-training sample efficiency in the Multi-Room environment with one-hot encodings. The ``Amortized Overhead'' represents the pre-training cost divided by the number of potential goal locations (states), indicating a constant scaling factor relative to environment size.}
    \label{tab:overhead_scaling}
    \vspace{0.2cm} 
    \begin{tabular}{cccc}
        \toprule
        \textbf{Number of Rooms} & \textbf{State Space Size} ($|S|$) & \textbf{Pre-training Steps} ($T$) & \textbf{Amortized Overhead} ($T/|S|$) \\
        \midrule
        3 & 43  & $2800 \pm 285$  & 65 \\
        4 & 65  & $4100 \pm 235$  & 63 \\
        5 & 109 & $7800 \pm 380$  & 72 \\
        6 & 130 & $9600 \pm 335$  & 74 \\
        7 & 153 & $10800 \pm 425$ & 70 \\
        \bottomrule
    \end{tabular}
\end{table}

This preliminary analysis suggests that the additional cost incurred by pre-training is relatively minor when amortized across the distribution of possible tasks. By investing a linear amount of exploration upfront, we enable rapid adaptation to any specific goal, providing a sample efficiency boost that outweighs the initial overhead compared to training from scratch for each new target.

While these results are specific to the tabular setting with exact state representations, they offer a promising lower bound on efficiency. In continuous domains, where the number of potential goal locations is theoretically infinite, the amortized overhead approaches zero ($T / \infty \to 0$). However, practical applications often constrain goals to specific manifolds or subsets; in such cases, this linear scaling implies that \algo{} remains a computationally viable strategy for initializing goal-agnostic agents.

\subsection{Replay Buffer transfer to Offline algorithms}
\label{sec:data_transfer}

Transferring the replay buffer collected during unsupervised pre-training and transferring the pre-trained policy are distinct evaluation protocols, and they need not produce the same ranking across exploration methods. Policy transfer measures whether pre-training provides a good initialization for subsequent online learning, whereas data transfer measures the quality of the dataset collected during pre-training when reused by an offline RL algorithm \cite{laskin2021urlb,yarats2022dontchangealgorithmchange}. We therefore complement the policy-transfer results with a separate data-transfer study, even if it is not the main purpose of \algo.

Following \citet{yarats2022dontchangealgorithmchange}, we collect a fixed dataset with each exploration method and train the downstream agent purely offline, without any further interaction with the environment. As downstream learners, we consider CQL \cite{kumar2020conservativeqlearningofflinereinforcement}, which is explicitly designed for offline RL through conservative regularization, and DDPG \cite{lillicrap2015ddpg}, which we report only as a reference since it lacks any mechanism to control extrapolation error in the offline setting. We evaluate the resulting policies deterministically after $15$k gradient updates. All other aspects of the setup are the same as in \cref{sec:tabular-exps}, considering the \textit{Two Rooms} and \textit{Multi Rooms} settings.

The results are reported in \cref{tab:offline_transfer}. Under CQL, \textsc{ROVER} is consistently the strongest method, matching or outperforming all baselines in every environment and remaining the only method that succeeds on the hardest Multi-Room task (Goal: Room~5). Under offline DDPG, the picture is different: \textsc{ROVER} remains competitive and performs best on the easier settings, whereas CIC attains stronger average performance on the most distant goals.

This difference can be explained by the shape of the replay buffer distribution. As shown in the \cref{app:heatmaps}, CIC and \textsc{ROVER} cover similar regions of the state space, but differ in density. CIC produces a more homogeneous dataset, while \textsc{ROVER} places more mass near the initial states and along trajectory prefixes, consistent with its objective of exploring new regions without forgetting previously visited ones. Compared to CIC and \textsc{ROVER}, APT, RND, and SMM appear less effective at producing replay buffers that support downstream offline learning on these navigation tasks.

The same dataset structure affects the two offline algorithms differently. Offline DDPG relies on repeated batch updates and tends to benefit from more homogeneous coverage, which reduces instability in poorly represented regions. This helps explain why CIC can be stronger on the hardest goals under offline DDPG. CQL, by contrast, prefers transitions that are well supported by the data. Thus, the dataset need not be uniform everywhere: it is more important that it contains reliable start-to-goal transition chains. Even if the goal itself is visited less often, repeated intermediate transitions can still support near-optimal CQL performance. This is the kind of structure induced by \textsc{ROVER}, which explains why its data are preferred by CQL.

Overall, these results show that data transfer and behavior transfer capture different properties of unsupervised exploration, while also confirming that the replay buffers generated by \textsc{ROVER} remain highly effective in the fully offline setting.

\label{app:offline_transfer}
\begin{table*}[h]
\centering
\caption{Offline RL performance after $15$k gradient updates on custom goal-navigation tasks. CQL and DDPG are trained on fixed datasets collected during unsupervised pre-training, without further environment interaction. Results are mean $\pm$ standard error over $7$ seeds. Goals are in Rooms 1--2 for Two Rooms and Rooms 2, 4, 5 for Multi Rooms. Online DDPG trained for $50$k interactions is reported as a reference. Best result is boldfaced separately for each algorithm.}
\label{tab:offline_transfer}
\resizebox{\textwidth}{!}{%
\begin{tabular}{llccccc}
\toprule
\multirow{2}{*}{Algorithm} & \multirow{2}{*}{Dataset policy} 
& \multicolumn{2}{c}{Two Rooms} 
& \multicolumn{3}{c}{Multi Rooms} \\
\cmidrule(lr){3-4} \cmidrule(lr){5-7}
& & Goal: Room 1 & Goal: Room 2 & Goal: Room 2 & Goal: Room 4 & Goal: Room 5 \\
\midrule
\multirow{6}{*}{CQL}
& Random & $-11.0\,\pm\,0.0$ & $-35.0\,\pm\,37.3$ & $ -300.0 \,\pm\, 0.0 $ & $ -300.0 \,\pm\, 0.0 $ & $ -300.0 \,\pm\, 0.0 $  \\
& APT    & $-11.0\,\pm\,0.0$ & $-80.0\,\pm\,0$    & $-11.0\,\pm\,0.0$ & $ -300.0 \,\pm\, 0.0 $ & $ -300.0 \,\pm\, 0.0 $  \\
& SMM    & $-11.0\,\pm\,0.0$ & $-80.0\,\pm\,0.0$  & $ -300.0 \,\pm\, 0.0 $ & $ -300.0 \,\pm\, 0.0 $ & $ -300.0 \,\pm\, 0.0 $  \\
& CIC    & $-11.0\,\pm\,0.0$ & $\mathbf{-17.0\,\pm\,0.0}$  & $-11.0\,\pm\,0.0$ & $-21.0\,\pm\,0.0$ & $ -300.0 \,\pm\, 0.0 $  \\
& RND    & $-11.0\,\pm\,0.0$ & $-80.0\,\pm\,0.0$  & $-11.0\,\pm\,0.0$ & $ -300.0 \,\pm\, 0.0 $ & $ -300.0 \,\pm\, 0.0 $  \\
& ROVER (Ours) &  $-11.0\,\pm\,0.0$ & $\mathbf{-17.0\,\pm\,0.0}$ & $\mathbf{-11.0\,\pm\,0.0}$ & $\mathbf{-21.0\,\pm\,0.0}$  & $\mathbf{-26.0\,\pm\,0.0
}$ \\
\midrule
\multirow{6}{*}{DDPG}
& Random & $-11.0\,\pm\,0.0$ & $-32.0\,\pm\,62.4$ & $-218.2\,\pm\,140.0$ & $ -300.0 \,\pm\, 0.0 $ & $ -300.0 \,\pm\, 0.0 $  \\
& APT    & $-11.0\,\pm\,0.0$ & $-80.0\,\pm\,0.0$  & $-135.1\,\pm\,154.5$ & $ -300.0 \,\pm\, 0.0 $ & $ -300.0 \,\pm\, 0.0 $  \\
& SMM    & $-11.0\,\pm\,0.0$ & $-80.0\,\pm\,0.0$  & $-300.0\,\pm\,0.0$ & $ -300.0 \,\pm\, 0.0 $ & $ -300.0 \,\pm\, 0.0 $  \\
& CIC    & $-11.0\,\pm\,0.0$ & $-71.2\,\pm\,23.8$ & $-91.6\,\pm\,141.0$ & $\mathbf{-100.7\,\pm\,136.1}$ & $\mathbf{-143.4\,\pm\, 146.5}$ \\
& RND    & $-11.0\,\pm\,0.0$ & $-80.0\,\pm\,0.0$ & $-217.4\,\pm\,141.3$ & $ -300.0 \,\pm\, 0.0 $ & $ -300.0 \,\pm\, 0.0 $   \\
& ROVER (Ours)  & $-11.0\,\pm\,0.0$ & $\mathbf{-17.0\,\pm\,0.0}$ &  $\mathbf{-54.4\,\pm\,105.9}$ & $-140.6\,\pm\, 149.1$ & $-221.7\,\pm\,133.7$  \\
\midrule
\multirow{1}{*}{DDPG (online)}
& 50k interactions & $-11.0\,\pm\,0.0$ & $-17.0\,\pm\,30.1$ & $-11.0\,\pm\,0.0$ & $-160.5\,\pm\,152.8$ & $-221.3\,\pm\,133.7$ \\
\bottomrule
\end{tabular}%
}
\end{table*}

\subsubsection{Replay Buffer Coverage Heatmaps}\label{app:heatmaps}
To better understand the differences observed in the data-transfer experiments, we visualize the empirical state-visitation distributions induced by the replay buffers collected during unsupervised pre-training. For each exploration method, we aggregate the visited states stored in the replay buffer and report the corresponding visitation heatmap over the grid. These plots provide a qualitative view of how the dataset is distributed spatially, complementing the return values reported in Table~\ref{tab:offline_transfer}.

The main observation is that the datasets produced by the different methods often have similar \emph{support}, in the sense that they eventually cover comparable portions of the navigable state space, but differ substantially in \emph{density}. In particular, CIC tends to produce a more homogeneous replay buffer once distant regions are reached, whereas \textsc{ROVER} places more mass near the initial states and along trajectory prefixes and bottlenecks. This behavior is consistent with the design of \textsc{ROVER}: the objective is not only to expand toward novel regions, but also to do so without forgetting previously visited ones. As a result, the replay buffer retains stronger density around the starting region and around the corridors that connect the start distribution to distant rooms.

This distinction helps explain the different behavior of the offline learners. A more homogeneous dataset, such as the one induced by CIC, can be beneficial to offline DDPG, whose batch updates are sensitive to uneven coverage and can deteriorate when large parts of the state-action space are weakly represented. By contrast, CQL benefits more from datasets that contain reliable and repeatedly observed transition chains from the start toward the goal. In this case, the denser prefix structure induced by \textsc{ROVER} is advantageous, since the conservative update can propagate value through well-supported transitions without relying on poorly represented actions. The remaining baselines are generally weaker in both respects: Random lacks directed exploration, while APT, RND, and SMM either fail to maintain sufficient support on long trajectories or produce datasets that are too sparse in the regions that matter for downstream control.

\begin{figure*}[t]
    \centering
    \includegraphics[width=\textwidth]{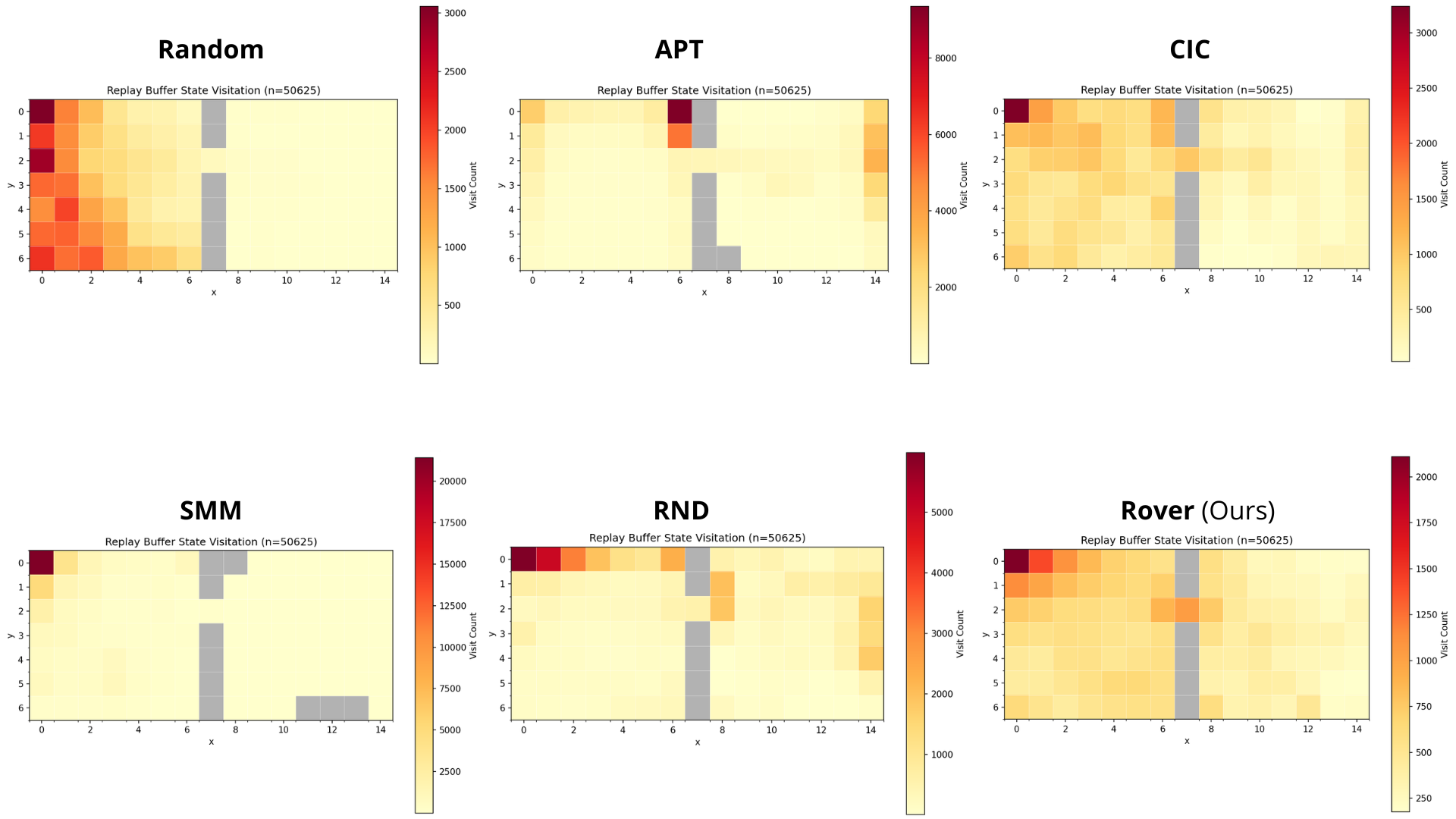}
  
    \caption{
    Replay-buffer state-visitation heatmaps in the \textbf{Two Rooms} environment for the different unsupervised exploration methods. Darker colors indicate states that appear more frequently in the replay buffer collected during pre-training. While several methods eventually cover both rooms, the induced datasets differ markedly in density. 
    }
    \label{fig:two_rooms_heatmaps}
\end{figure*}

\begin{figure*}[t]
    \centering
    \includegraphics[width=\textwidth]{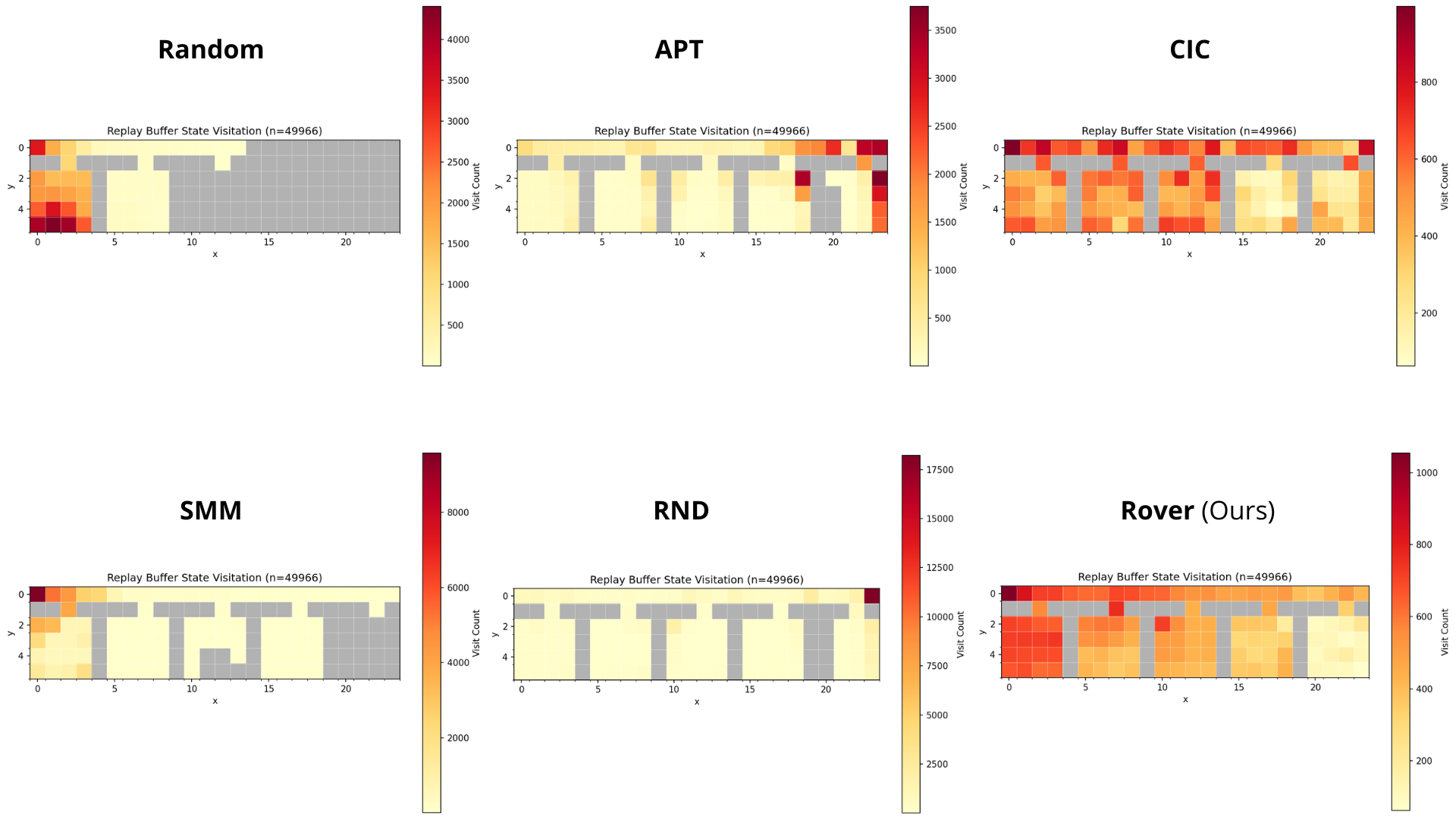}
    \caption{
    Replay-buffer state-visitation heatmaps in the \textbf{Multi Rooms} environment for the different unsupervised exploration methods. Although CIC and \textsc{ROVER} both reach distant rooms, their replay buffers differ in how often the traversed regions are revisited. 
    }
    \label{fig:multi_rooms_heatmaps}
\end{figure*}

Overall, the heatmaps support the interpretation given in \cref{sec:data_transfer}: in these navigation tasks, offline performance depends not only on whether a method reaches distant parts of the environment, but also on how the replay buffer distributes mass across start states, bottlenecks, and distal regions. This clarifies why replay buffers with similar support can nevertheless induce different behavior when paired with different offline RL algorithms.

\subsection{PointMaze Implementation Details}
\label{app:pointmaze_implementation}

\paragraph{Environment and action interface}
We pretrain the policies using ROVER or the baselines in \texttt{PointMaze\_UMaze-v3}, with the start fixed, without positional noise, at one of $(-1,1)$, $(1,0)$, or $(1,1)$. Its four discrete actions apply unit forces along the positive or negative $x$ and $y$ axes.  State observations contain $(x,y,v_x,v_y)$. Pixel observations are grayscale images resized to $84\times84$ and stacked over four frames.

\paragraph{Representations and kernels}
For state observations, a trainable MLP encoder maps the four-dimensional state through a 1,024-unit hidden layer to 256 features; ReLU outputs are $L^1$-normalized, and we train it using \cref{eq:l_nce}. ROVER applies a Gaussian dynamics kernel to these learned features and a Gaussian coverage kernel to raw $(x,y)$ positions. 

For pixel observations, a trainable CNN maps the four-frame stack to 1,024 features for an inner-product dynamics kernel. Its four convolutional layers use 32 channels, $3\times3$ kernels, and ReLU activations; the first layer has stride two. Adaptive $7\times7$ pooling is followed by a 1,024-dimensional linear projection, layer normalization, ReLU, and $L^1$ normalization, and is trained using \cref{eq:l_nce}. 
Pixel coverage uses the current grayscale frame through a separate frozen, randomly initialized CNN: its pooled convolutional maps are flattened to 1,568 features and normalized by their sum, then compared with a Gaussian kernel. Gaussian bandwidths are not fixed: the coverage bandwidth uses the median pairwise feature distance in the batch. We keep the coverage CNN encoder frozen during training since its convolutional architecture encodes translation equivariance \cite{bronstein2021geometricdeeplearninggrids}.

\begin{figure}[t]
\centering
\includegraphics[width=\linewidth]{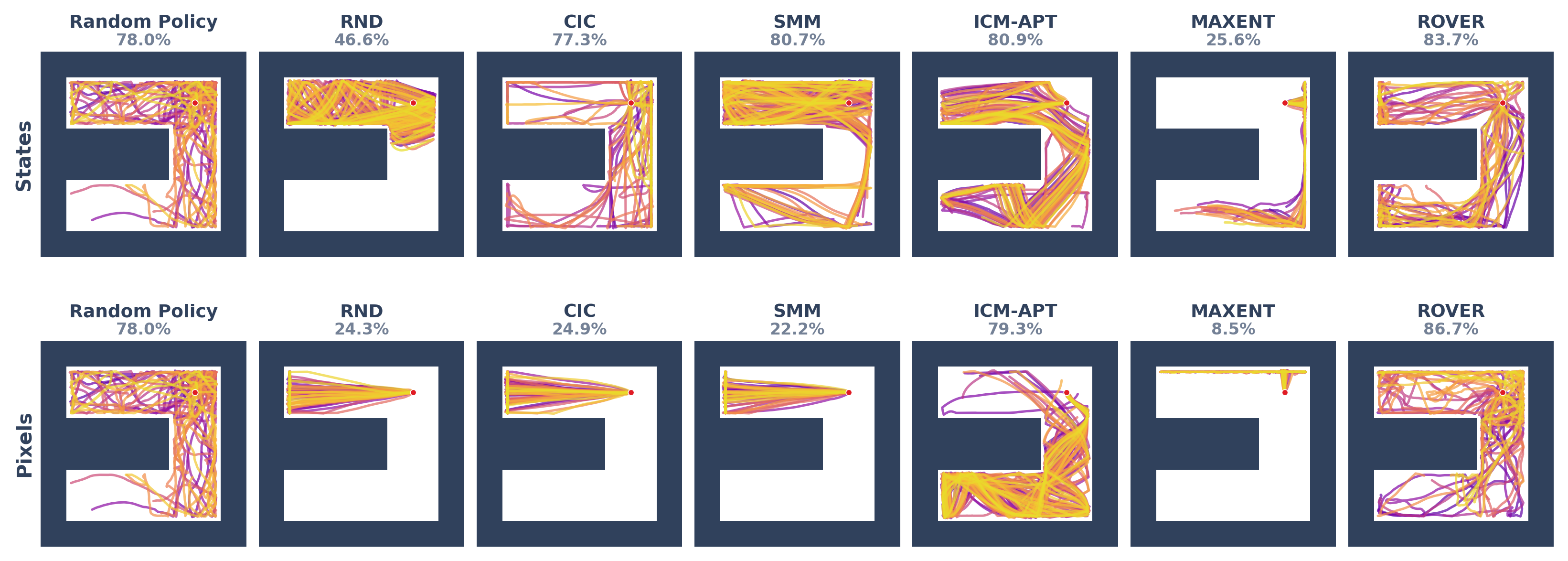}

\caption{Example position-coverage visualizations from 50 trajectories sampled from each stochastic policy. Percentages shown in the visualization describe the displayed sample and need not equal the aggregate means in \cref{tab:pointmaze_coverage}. \textbf{Top:} Pretrained policies with state observations. \textbf{Bottom:} Pretrained policies with pixel observations.
}
\label{fig:pointmaze_coverage}
\end{figure}

\newpage

\subsection{ROVER Initialization in Dense Reward Settings}

We further investigate the impact of ROVER pre-training in environments with dense, shaped rewards, defined as the negative distance to the target. In this setting, exploration is no longer ``blind" but explicitly guided by the reward gradient. We evaluate performance on the Multi-Room environment (5 rooms, goal in the 4th), comparing both discrete (one-hot) and visual (pixel-based) observations to decouple the effects of the pre-trained policy from the pre-trained encoder.

In the discrete setting (\cref{fig:dense}, Left), ROVER initialization provides no significant advantage over learning from scratch. This is expected, as the dense reward signal is sufficient to guide the agent without the need for intrinsic exploration. However, in the visual domain (\cref{fig:dense}, Right), transferring the pre-trained encoder yields a substantial acceleration in convergence. This suggests that while the policy prior offers limited utility when dense guidance is available, the representation learned by the world model remains highly effective for visual feature extraction. We leave a granular decomposition of these effects—isolating the encoder transfer from the policy transfer—to future work.

\begin{figure}[h!]
    \centering
    \includegraphics[width=1\linewidth]{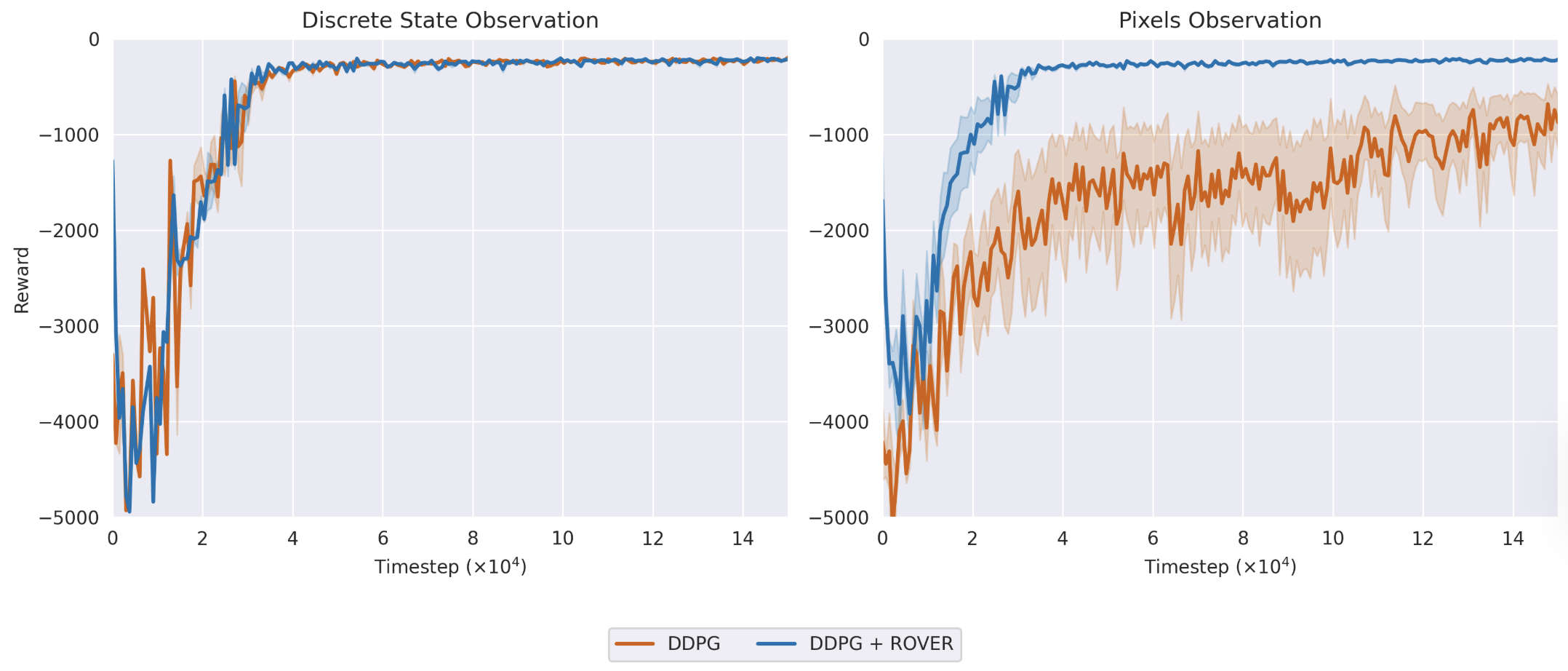}
    \caption{In this plot, we show the learning curves using DDPG and DDP + \algo\ in state-based (left) and pixel-based (right) environments. In particular, we considered \textit{Multi-Room}, where the goal is in the fourth room.}
    \label{fig:dense}
\end{figure}

\end{document}